\documentclass[%
onecolumn 
, colorlinks 
]{mpi2015-cscpreprint}

\usepackage[american]{babel}

\usepackage{graphicx}
\usepackage{bm}
\usepackage{amssymb}
\usepackage{amsthm}
\usepackage{amsmath}
\usepackage{isomath}
\usepackage[vlined,longend,linesnumbered,ruled]{algorithm2e}
\usepackage{mathabx}
\usepackage{caption}
\usepackage{subcaption}
\usepackage[colorinlistoftodos,prependcaption]{todonotes}

\usepackage{comment}
\usepackage{cleveref}
\newcommand{\R}{\mathbb{R}}
\newcommand{\N}{\mathbb{N}}
\newcommand{\C}{\mathbb{C}}
\newcommand{\T}{\mathbb{T}}
\newcommand{\norm}[1]{\left\lVert#1\right\rVert}

\newcommand{\strich}{\,\middle\vert\,}

\newcommand{\pq}{\begin{bmatrix}\ve{q} \\ \ve{p} \end{bmatrix}}
\newcommand{\ve}[1]{\vectorsym{#1}}
\newcommand{\ma}[1]{\matrixsym{#1}}
\newcommand{\te}[1]{\tensorsym{#1}}
\newcommand{\vect}[1]{\texttt{vec}(#1)}

\DeclareMathOperator*{\diag}{diag}

\DeclareMathOperator{\Diff}{d}
\DeclareMathOperator*{\pool}{MaxPool}
\DeclareMathOperator*{\unpool}{MaxUnPool}

\DeclareMathOperator{\sech}{sech}

\theoremstyle{definition}
\newtheorem{defi}{Definition}
\newtheorem{remark}{Remark}
\crefname{defi}{Definition}{Definitions}
\newtheorem{prop}{Proposition}

\begin{document}
	
	
	\title{Symplectic convolutional neural networks}

    \author[$\ast$]{Süleyman Y\i ld\i z}
    \affil[$\ast$]{Max Planck Institute for Dynamics of Complex Technical Systems, 39106 Magdeburg, Germany.\authorcr
    	\email{yildiz@mpi-magdeburg.mpg.de}, \orcid{0000-0001-7904-605X}
    }
    
    \author[$\ast\ast$]{Konrad Janik}
    \affil[$\ast\ast$]{Max Planck Institute for Dynamics of Complex Technical Systems, 39106 Magdeburg, Germany.\authorcr
    	\email{janik@mpi-magdeburg.mpg.de}, \orcid{0009-0004-9030-0708}
    }
    
    \author[$\dagger\ddagger$]{Peter Benner}
    \affil[$\dagger$]{Max Planck Institute for Dynamics of Complex Technical Systems, 39106 Magdeburg, Germany.\authorcr
    	\email{benner@mpi-magdeburg.mpg.de}, \orcid{0000-0003-3362-4103}
    }
    \affil[$\ddagger$]{Otto von Guericke University,  Universit\"atsplatz 2, 39106 Magdeburg, Germany\authorcr
    	\email{peter.benner@ovgu.de} 
    	\vspace{-0.5cm}
    }

	\shorttitle{Symplectic CNN}
	\shortauthor{S. Y\i ld\i z, K. Janik, P. Benner}
	\shortdate{}
	

\abstract{%
	
We propose a new symplectic convolutional neural network (CNN) architecture by leveraging symplectic neural networks, proper symplectic decomposition, and tensor techniques. Specifically, we first introduce a mathematically equivalent form of the convolution layer and then, using symplectic neural networks, we demonstrate a way to parameterize the layers of the CNN to ensure that the convolution layer remains symplectic. To construct a complete autoencoder, we introduce a symplectic pooling layer. 
We demonstrate the performance of the proposed neural network on three examples: the wave equation, the nonlinear Schrödinger (NLS) equation, and the sine-Gordon equation. The numerical results indicate that the symplectic CNN outperforms the linear symplectic autoencoder obtained via proper symplectic decomposition. 
}

\novelty{ 
    \begin{itemize} 
	\item A symplectic convolutional autoencoder is proposed. 
	\item The performance of the proposed autoencoder is tested on several numerical examples, and compared to the proper symplectic decomposition-based autoencoder. \end{itemize} 
}

\maketitle
\section{Introduction}
Over the past few decades, the increasing power of computer hardware has transformed deep learning into a powerful tool for solving a wide range of real-world problems. Among the many techniques in deep learning, autoencoders have emerged as a crucial component. Various deep learning approaches heavily rely on autoencoders for tasks such as automatic speech recognition~\cite{hinton2012deep,hannun2014deep}, computer vision~\cite{lecun1989backpropagation, krizhevsky2017imagenet,simonyan2014very}, and natural language processing~\cite{collobert2011natural,mikolov2013distributed,cho2014learning,szegedy2015going}. Autoencoders provide generalized representations of the underlying systems~\cite{bengio2013representation}. This capability is particularly important as it allows for the extraction of meaningful features from complex data, facilitating more accurate and efficient analysis. Early usage of autoencoders mainly focused on dimensionality reduction~\cite{hinton2006reducing}. This application is essential for model order reduction (MOR) methods, which focus on constructing low-dimensional models from complex high-fidelity data while retaining their core characteristics. For a detailed overview of MOR methods, we refer to the handbook~\cite{morBenGQetal21,morBenGQetal21a,morBenGQetal21b}. 

Dimension reduction methods can be classified into two main categories, namely linear and nonlinear methods~\cite{hou2022dimensionality}. One of the very popular linear dimensionality reduction methods based on using principal component analysis (PCA)~\cite{wold1987principal}, also known as proper orthogonal decomposition (POD)~\cite{berkooz1993proper}. Despite their popularity, linear methods fail to provide comparable approximations to nonlinear dimension reduction method for highly nonlinear problems~\cite{hinton2006reducing}. The efficiency and stability of these methods depend on the specific use case. For example, when constructing reduced-order models for canonical Hamiltonian systems, one property to be enforced on the autoencoder might be symplecticity. This automatically preserves the Hamiltonian dynamics in the latent space \cite{hairer2006structure} and it is important because the latent dynamics inherit the desirable properties of Hamiltonian systems, such as long-term stability \cite{PenMoh16}. 

In this paper, we present a symplectic autoencoder method that can be applied to various problems. However, since its primary application is dimensionality reduction for Hamiltonian dynamics, we focus on the relevant literature. Dimensionality reduction for Hamiltonian dynamics becomes essential when learning Hamiltonian dynamics with high-fidelity data arising from the discretization of partial differential equations. Learning the dynamics of Hamiltonian systems using high-dimensional data might be infeasible, even with high-performance computing (HPC) machines. On the other hand, for small-dimensional systems, there are several works that can be used to learn the dynamics of Hamiltonian systems, such as Hamiltonian neural networks \cite{greydanus2019hamiltonian}, symplectic neural networks \cite{JinZZ20, JanB25}, Bayesian system identification \cite{galioto2020bayesian}, and Gaussian processes \cite{bertalan2019learning}. To deal with high-dimensional data in \cite{PenMoh16}, a linear symplectic autoencoder that uses proper symplectic decomposition (PSD) is introduced for model order reduction (MOR) and compared with the POD-Galerkin method. It is shown that the surrogate model obtained with PSD outperforms the POD-Galerkin model. Using the linear symplectic autoencoder, which we refer to as the PSD autoencoder, \cite{sharma2022hamiltonian} introduced a data-driven, non-intrusive reduced-order model known as Hamiltonian operator inference. Moreover, there are several linear subspace models for structure-preserving MOR method for Hamiltonian systems \cite{buchfink2020psd,afkham2017structure,bajars2025structure}. Nevertheless, in most cases, linear subspace methods require relatively large dimensions to achieve acceptable approximations of Hamiltonian dynamics. In general, autoencoders are neural networks (NNs) that utilize traditional architectures such as multilayer perceptrons (MLPs), convolutional neural networks (CNNs), or recurrent neural networks (RNNs)~\cite{hou2022dimensionality}. In~\cite{BraKra23}, the authors propose a nonlinear symplectic autoencoder with a MLP architecture. This approach leverages techniques from symplectic neural networks (SympNets)~\cite{JinZZ20} and PSD methods. However, autoencoders based on MLPs typically require significantly more parameters compared to those based on CNNs. A large portion of popular autoencoder algorithms depends on CNNs, which have become a cornerstone in the field. With their wide range of successful applications, CNNs have proven to be highly effective tools. Some recent works have studied Hamiltonian dynamics by weakly enforcing symplecticity on CNN based autoencoders, i.e. by minimizing the resiudal of the symplecicity condition through a loss function~\cite{buchfink2023symplectic,yildiz2024data,goyal2025deep}. 

In this paper, we present a general framework for constructing symplectic  CNN based autoencoders. Specifically, we integrate the methods from~\cite{BraKra23} and~\cite{JinZZ20} with tensor techniques to develop symplectic convolutional autoencoders. Typically a convolutional autoencoder consists of four different types of layers: convolutional layers, pooling layers, activation layers and fully connected layers. We enforce symplecticity on the convolutional and activation layers by utilizing the concepts from SympNets \cite{JinZZ20}. Additionally, we use the approach described in \cite{BraKra23} to construct PSD-like pooling and fully connected layers.

The remainder of the paper is organized as follows: \Cref{sec:conv} introduces the equivalent form of convolutional neural networks using tensor techniques. \Cref{sec:SympCNN} describes symplectic neural networks and proper symplectic decomposition, along with their application in constructing symplectic convolutional autoencoders (SympCAEs). \Cref{sec:num} demonstrates the accuracy of the proposed method with 1D and 2D test cases. Finally, we provide concluding remarks and discuss future directions in \Cref{sec:conc}.
 
\section{Convolutional neural networks}\label{sec:conv}

To explain the basic idea behind the equivalent mathematical form of CNN, we exploit the vectorization of the input signals. Let us first consider a 1D input signal $\te{x}\in \R^{1 \times C_\text{in} \times N}$ of length $N$ with the number of input channels $C_\text{in}$. For simplicity, we only consider convolutions with zero padding, stride size of 1, dilation of 1 and without bias. Let us define the vectorization operator with $\vect{\cdot}$, which creates a column vector from a matrix by stacking the columns of the matrix into a column vector. Moreover, we denote $\vect{\te{x}}=\ve{x}\in \R^{N C_\text{in}}$ as the vectorization of the input signal, $C_\text{out}$ as the number of output channels, $\te{w}\in \R^{C_\text{out} \times C_\text{in} \times l}$ as the weights of the convolutional layer, $\te{A}_{i,j,k}$ as the $(i,j,k)$th element of the tensor $\te{A}$ and $l$ as the length of the convolutional weight. To describe the mathematically equivalent formulation of the convolution operation, first let us define the following Toeplitz matrices,
\begin{equation}{\label{eqn:1dToep}}
\ma{T}_{i,j}=
\begin{bmatrix}
\te{w}_{i,j,m}&\ldots&\te{w}_{i,j,l}& & & & \vspace{0.6em}\\
\vdots & \ddots & &\ddots& & & \vspace{0.6em} \\ 
\te{w}_{i,j,1}&\ldots &\te{w}_{i,j,m}&\ldots&\te{w}_{i,j,l} & & \vspace{0.6em} \\
&\ddots& &\ddots& &\ddots & \vspace{0.6em}\\
& &\te{w}_{i,j,1} &\ldots &\te{w}_{i,j,m}&\ldots &\te{w}_{i,j,l}  \vspace{0.6em} \\
& &  &\ddots &  &\ddots  & \vspace{0.6em} \\
& & & &\te{w}_{i,j,1}&\ldots &	\te{w}_{i,j,m}
\end{bmatrix} \in \R^{N \times N}
\end{equation}
for $i=1,\ldots,C_\text{out}, ~j=1,\ldots,C_\text{in}$, and $m=(l+1)/2$ for odd kernel length $l$. This choice of $T_{i,j}$ corresponds to zero-padding and padding size $(l-1)/2$ to keep the input and output dimensions the same. Hence, the kernel length $l$ has to be odd, if we want to preserve the dimensions and work with square Toeplitz matrices.
Using \cref{eqn:1dToep}, we can write the equivalent mathematical formulation of the 1D convolution layer  \cite{gilbert2017towards}  as follows:
\begin{equation}{\label{eqn:1dvecConv}}
\begin{bmatrix}
\ma{T}_{1,1}&\ldots&\ma{T}_{1,C_\text{in}}\\
\vdots &\ddots & \vdots \\
\ma{T}_{C_\text{out},1}&\ldots&\ma{T}_{C_\text{out},C_\text{in}}\\
\end{bmatrix}
\ve{x}=\ve{y},
\end{equation}
where $\ve{y}$ is the vectorization of the output signal $\te{y}$.

Similarly, for the equivalent mathematical formulation of the 2D convolution layer we can exploit \cref{eqn:1dvecConv}, by only changing the definition of the Toeplitz matrices $\ma{T}_{ij}$. 
To construct a  2D convolutional layer with Toeplitz matrices, let us first consider $\te{x}\in \R^{1 \times C_\text{in} \times N_1 \times N_2}$ as 2D input signal of size $N_1\times N_2$ and $\te{w}$ as the weights of the 2D convolutional layer. In the following, we use $\ma{T}\in \R^{C_\text{out}N_1 N_2 \times C_\text{in} N_1 N_2}$ for Toeplitz matrices also for 2D convolutional layers because the construction is quite similar. Unless explicitly mentioned otherwise, $\ma{T}$ is going to refer the Toeplitz matrices of 1D or 2D convolutional layer depending on the context it has been used. Moreover, let us define the following matrix
\begin{equation}{\label{eqn:2dToepsub}}
	\ma{\tau}_{i,j,k}=
	\begin{bmatrix}
		\te{w}_{i,j,m_1,k}&\ldots&\te{w}_{i,j,l_1, k}& & & & \vspace{0.8em}\\
		\vdots & \ddots & &\ddots& & & \vspace{0.8em} \\ 
		\te{w}_{i,j,1,k}&\ldots &\te{w}_{i,j,m_1, k}&\ldots&\te{w}_{i,j,l_1, k} & & \vspace{0.8em} \\
		&\ddots& &\ddots& &\ddots & \vspace{0.8em}\\
		& &\te{w}_{i,j,1,k} &\ldots &\te{w}_{i,j,m_1,k}&\ldots &\te{w}_{i,j,l_1,k}  \vspace{0.8em} \\
		& &  &\ddots &  &\ddots  & \vspace{0.8em} \\
		& & & &\te{w}_{i,j,1, k}&\ldots &	\te{w}_{i,j,m_1,k}
	\end{bmatrix}\in \R^{N_1 \times N_1}.
\end{equation}
Using \cref{eqn:2dToepsub}, we can define the matrices $\ma{T}_{i,j}$ for a 2D convolutional layer by replacing the elements of \cref{eqn:1dToep} by $\ma{\tau}_{i,j,k}$, i.e., changing the entries $\te{w}_{i,j,m}$ with matrices $\ma{\tau}_{i,j,k}$. This way, we define block-Toeplitz matrices for 2D case as follows
\begin{equation}{\label{eqn:2dToep}}
	\ma{T}_{ij}=
	\begin{bmatrix}
		\te{\tau}_{i,j,m_2}&\ldots&\te{\tau}_{i,j,l_2}& & & & \vspace{0.8em}\\
		\vdots & \ddots & &\ddots& & & \vspace{0.8em} \\ 
		\te{\tau}_{i,j,1}&\ldots &\te{\tau}_{i,j,m_2}&\ldots&\te{\tau}_{i,j,l_2} & & \vspace{0.8em} \\
		&\ddots& &\ddots& &\ddots & \vspace{0.8em}\\
		& &\te{\tau}_{i,j,1} &\ldots &\te{\tau}_{i,j,m_2}&\ldots &\te{\tau}_{i,j,l_2}  \vspace{0.8em} \\
		& &  &\ddots &  &\ddots  & \vspace{0.8em} \\
		& & & &\te{\tau}_{i,j,1}&\ldots &	\te{\tau}_{i,j,m_2}
	\end{bmatrix}\in \R^{N_1 N_2 \times N_1 N_2}.
\end{equation}
Finally, analogous to \cref{eqn:1dvecConv} we can represent a 2D convolutional layer by using \cref{eqn:2dToep} as 
 \begin{equation*}
 	\begin{bmatrix}
 		\ma{T}_{1,1}&\ldots&\ma{T}_{1,C_\text{in}}\\
 		\vdots &\ddots & \vdots \\
 		\ma{T}_{C_\text{out},1}&\ldots&\ma{T}_{C_\text{out},C_\text{in}}\\
 	\end{bmatrix}
 	\ve{x}=\ve{y},
 \end{equation*}
 where $\ve{y}$ is the vectorization of the output signal $\te{y}$.
\begin{remark}
	Using the same construction technique in 1D and 2D case, it is possible to extend the equivalent formulation to the 3D case.
\end{remark}
\begin{remark}
 We only use the mathematically equivalent form of the convolution operation for theoretical contributions, in the experiments we use standard convolution operators provided by \texttt{PyTorch} \cite{PasGMetal19}, i.e. we do not construct these matrices.
\end{remark}
\section{Symplectic CNNs}\label{sec:SympCNN}\label{subsec:SympNets}

In this section, we define a symplectic convolutional autoencoder by integrating ideas from SympNets~\cite{JinZZ20} and the symplectic autoencoder~\cite{BraKra23}. We begin with a brief overview of Hamiltonian systems and describe symplectic lifting and reduction. Then, we summarize SympNets~\cite{JinZZ20} and demonstrate how these ideas can be integrated to construct a convolutional autoencoder. 

Let us first denote $\ma{0}\in \R^{n \times n}$ as the matrix of zeros, $\ma{I}_n\in \R^{n \times n}$ as the identity matrix, $\ve{x}_0$ as the initial condition, and $\nabla_\ve{x}$ as the gradient with respect to $\ve{x}$.
Canonical Hamiltonian systems are defined by: 
\begin{equation} \label{eq:HamiltonianEquations} 
	\dot{\ve{x}}(t) = \ma{J}_{2n} \nabla_\ve{x} H(\ve{x}(t)) \in \R^{2n}, \quad \ve{x}(0) = \ve{x}_0, 
\end{equation} 
where $\ma{J}_{2n}$ is the canonical Poisson matrix of the form 
\begin{equation}\label{eqn:Jn} 
	\ma{J}_{2n} := 
	\begin{bmatrix} \ma{0} & \ma{I}_n\\ 
		-\ma{I}_n & \ma{0} 
	\end{bmatrix} \in \R^{2n \times 2n}, 
\end{equation} and the state $\ve{x} \in \R^{2n}$ contains generalized momenta $\ve{p} \in \R^n$ and generalized positions $\ve{q} \in \R^n$. The dynamics of the Hamiltonian system \cref{eq:HamiltonianEquations} is determined by the Hamiltonian ``energy'' function $H \colon \R^{2n} \to \R$, which remains constant over time for $x(t)$ satisfying \cref{eq:HamiltonianEquations}, i.e.
$$
\dfrac{d}{d t} H(\ve{x}(t)) = \nabla_\ve{x} H (\ve{x}(t))^T \dot{\ve{x}}(t) = \nabla_\ve{x} H (\ve{x}(t))^T \ma{J}_{2n} \nabla_\ve{x} H(\ve{x}(t)) = 0.
$$
Before introducing other properties of Hamiltonian systems, let us define symplectic transformations for both linear and general nonlinear cases. Let $\mathbb{V}$ denote a vector space of dimension $2n$, and let $\Omega$ be the symplectic form on $\mathbb{V}$. The symplectic form $\Omega$ is an alternating bilinear, and nondegenerate form, $\Omega: \mathbb{V} \times \mathbb{V} \rightarrow \mathbb{R}$. Assuming $\mathbb{V} = \mathbb{R}^{2n}$, for all $\ve{\xi}, \ve{\nu} \in \mathbb{V}$, the symplectic form can be represented as follows: $$ \Omega(\ve{\xi}, \ve{\nu}) = \ve{\xi}^T \ma{J}_{2n} \ve{\nu}. $$
Moreover, $(\mathbb{V}, \Omega)$ is called a symplectic vector space. Let us introduce $(\mathbb{V}, \Omega)$ and $(\mathbb{W}, \omega)$ as two symplectic vector spaces with $\dim(\mathbb{V})=2n$, $\dim(\mathbb{W})=2k$, and $k\leq n$.
\begin{defi}[\cite{PenMoh16}]\label{eq:linearSymplecticity}
	A linear map $\mathcal{A}:\mathbb{W}\rightarrow\mathbb{V}$ is called \emph{symplectic lifting} if it preserves the symplectic structure:
	\begin{equation}
		\label{eqn:symplecticLin}
		\omega(\ve{z},\ve{w})=\Omega(\mathcal{A}(\ve{z}),\mathcal{A}(\ve{w})).
	\end{equation}
\end{defi}
In canonical coordinates, where we can represent $\mathcal{A}$ via a matrix $\ma{A}$, \cref{eqn:symplecticLin} is equivalent to the following condition:
\begin{equation}\label{eqn:symplecticMatrix}
	\ma{A}^T\ma{J}_{2n}\ma{A}=\ma{J}_{2k}.
\end{equation}
Additionally, a matrix satisfying the condition in \cref{eqn:symplecticMatrix} is called a symplectic matrix. Let us denote the set of all symplectic matrices $A \in \R^{2n \times 2k}$ by $Sp(2k, \R^{2n})$, which is referred to as the symplectic Stiefel manifold \cite{PenMoh16}.
\begin{defi}[\cite{PenMoh16}]\label{eq:symplecticInverse}
	The \emph{symplectic inverse} $\ma{A}^{+}$ of a symplectic matrix $\ma{A} \in Sp(2k, \R^{2n})$ is defined as:
	$$\ma{A}^{+}=\ma{J}^T_{k}\ma{A}^T\ma{J}_{2n}.$$
\end{defi}
In general, let us recall a definition of a symplectic nonlinear transformation from, e.g.,~\cite{yildiz2024data}.
\begin{defi}[\cite{Sil08}]
	A map $\boldsymbol{\psi}:\R^{2k} \to \R^{2n},~n\geq k$, is a \emph{symplectic transformation from $\R^{2k}$ to $\R^{2n}$} when the following condition is  fulfilled:
	\begin{equation}\label{eq:symplecticEmbedding}
		(\Diff \boldsymbol{\psi}_\ve{x})^T \ma{J}_{2n} \Diff \boldsymbol{\psi}_\ve{x} = \ma{J}_{2k}, \qquad \forall~\ve{x} \in \R^{2k},
	\end{equation}
	where 	$\Diff \boldsymbol{\psi}_\ve{x} \in \R^{2n \times 2k}$ is the Jacobian of $ \boldsymbol{\psi}$ with respect to $\ve{x}$. This implies that the Jacobian of the map with respect to the state $\ve{x}$ is satisfying \cref{eqn:symplecticMatrix}, i.e., the Jacobian is a symplectic matrix for all $\ve{x}$. Furthermore, we refer to it as a \emph{symplectic lifting} \cite{yildiz2024data}. 
\end{defi}
The notion of symplecticity plays two important roles in Hamiltonian dynamics. Firstly, the flow $\mathbf{F}^t$ of Hamiltonian systems is symplectic. The flow of a Hamiltonian system refers to the map that transforms the initial point to the corresponding solution of the system at time $t$, i.e., $\mathbf{F}^t(\mathbf{x}_0)=\mathbf{x}(t)$. This concept has been utilized in SympNets~\cite{JinZZ20} to learn the flow of Hamiltonian systems. Secondly, symplectic transformations, also known as canonical transformations, are used for coordinate transformations of Hamiltonian systems. A canonical transformation preserves Hamilton's equations, meaning it transforms Hamiltonian coordinates into another set of coordinates that also describe a Hamiltonian system.  For a more detailed overview of symplectic transformations and Hamiltonian systems, we refer to the book~\cite{marsden2013introduction}.

\subsection{SympNets}
Next, we summarize SympNets~\cite{JinZZ20} which focus on the flow of Hamiltonian systems. Let us first define the following notation for matrix-like nonlinear maps
\begin{align*}
	\begin{bmatrix}
		f_1 & f_2 \\
		f_3 & f_4
	\end{bmatrix}: \R^{2n} \to \R^{2n}, \qquad
	\begin{bmatrix}
		f_1 & f_2 \\
		f_3 & f_4
	\end{bmatrix}
	\begin{bmatrix}
		\ve{q}\\ \ve{p}
	\end{bmatrix}:=
	\begin{bmatrix}
		f_1(\ve{q})+f_2(\ve{p})\\
		f_3(\ve{q})+f_4(\ve{p})
	\end{bmatrix},
\end{align*}
where $f_i:\R^{n} \to \R^n,~i=1,...,4$.
\begin{remark}
	Note that the SympNets in \cite{JinZZ20} and \cite{JanB25} use $\ve{p}$ as the upper part of $\ve{x}$ instead of $\ve{q}$. Since the roles of $\ve{p}$ and $\ve{q}$ are interchagable in SympNets, we use $\ve{q}$ as the first part of $\ve{x}$ to keep the notation consistent throughout the paper. 
\end{remark}
There are two kinds of SympNets presented in \cite{JinZZ20}. We only focus on the so called LA-SympNets, since their architecture can make use of symplectic convolutional layers, while there is no obvious way to use them in the G-SympNet architecture, because they do not have linear layers as building blocks. \cite{JanB25} introduced a new time-adaptive form of linear and activation modules for LA-SympNets and showed that the resulting SympNets have the same approximation properties as the original SympNets from \cite{JinZZ20}, while using $n$ less parameters per layer. Hence, we are going to use the modules from \cite{JanB25} by fixing the adaptive time step $h=1$, because we just want to parametrize one symplectic map instead of a family of maps. As discussed in \cite{JanB25}, this also removes the necessity of the inverse linear layers.
\begin{defi}[LA-SympNet \cite{JanB25}]
	\label{def:SympNet}
	First we define \emph{linear modules} by
	\begin{align*}
		\mathcal{L}_m^{\text{up}}\begin{bmatrix}
			\ve{q}\\ \ve{p}
		\end{bmatrix}&=
		\begin{bmatrix}
			\mathbf{I}_n & \mathbf{0}/\mathbf{S}_m\\ \mathbf{S}_m/\mathbf{0} & \mathbf{I}_n
		\end{bmatrix}...\begin{bmatrix}
			\mathbf{I}_n & \mathbf{0}\\ \mathbf{S}_2 & \mathbf{I}_n
		\end{bmatrix}
		\begin{bmatrix}
			\mathbf{I}_n & \mathbf{S}_1\\ \mathbf{0} &\mathbf{I}_n
		\end{bmatrix}
		\begin{bmatrix}
			\ve{q}\\ \ve{p}
		\end{bmatrix},\\
		\mathcal{L}_m^{\text{low}}\begin{bmatrix}
			\ve{q}\\ \ve{p}
		\end{bmatrix}&=
		\begin{bmatrix}
			\mathbf{I}_n & \mathbf{0}/\mathbf{S}_m\\ \mathbf{S}_m/\mathbf{0} & \mathbf{I}_n
		\end{bmatrix}...\begin{bmatrix}
			\mathbf{I}_n & \mathbf{S}_2\\ \mathbf{0} & \mathbf{I}_n
		\end{bmatrix}
		\begin{bmatrix}
			\mathbf{I}_n & \mathbf{0}\\ \mathbf{S}_1 & \mathbf{I}_n
		\end{bmatrix}
		\begin{bmatrix}
			\ve{q}\\ \ve{p}
		\end{bmatrix}
	\end{align*}
	with $\mathbf{S}_1,...,\mathbf{S}_n \in \R^{n \times n}$ symmetric. We write
	\begin{align*}
		\mathcal{M}_\text{L}:=\left\{v \vert v \text{ is a linear module}\right\}
	\end{align*} for the \emph{set of linear modules}. For an activation function $\sigma$, the corresponding \emph{activation modules} are given by
	\begin{align*}
		\mathcal{N}_\text{up}\begin{bmatrix}
			\ve{q}\\ \ve{p}
		\end{bmatrix}=\begin{bmatrix}
			\mathbf{I}_n & \tilde{\sigma}_{a,b}\\ \mathbf{0} & \mathbf{I}_n
		\end{bmatrix}
		\begin{bmatrix}
			\ve{q}\\ \ve{p}
		\end{bmatrix}, \qquad
		\mathcal{N}_\text{low}\begin{bmatrix}
			\ve{q}\\ \ve{p}
		\end{bmatrix}=\begin{bmatrix}
			\mathbf{I}_n & \mathbf{0}\\ \tilde{\sigma}_{a,b} & \mathbf{I}_n
		\end{bmatrix}
		\begin{bmatrix}
			\ve{q}\\ \ve{p}
		\end{bmatrix}
	\end{align*}
	with $\tilde{\sigma}_{\ve{a},\ve{b}}(\ve{x}):=\diag(\ve{a})\sigma(\ve{x}+\ve{b}),$ where $\ve{a},\ve{b} \in \R^d.$ We denote the \emph{set of activation modules} by
	\begin{align*}
		\mathcal{M}_\text{A}=\left\{w \vert w \text{ is an activation module}\right\}.
	\end{align*}
	Now we can define the set of \emph{LA-SympNets} as
	\begin{align*}
		\Psi_{\text{LA}}:=\left\{\psi= v_{l+1} \circ w_k \circ v_l \circ ... \circ w_1 \circ v_1\vert v_1,...,v_{k+1}\in \mathcal{M}_\text{L}, w_1, ..., w_l \in \mathcal{M}_\text{A}, l \in \N\right\}
	\end{align*}
\end{defi}
The SympNets defined in \Cref{def:SympNet} are symplectic by design, because every layer $u \in \mathcal{M}_\text{L}\cup \mathcal{M}_\text{A}$ is of the form
\begin{align*}
		u\pq =
		\begin{bmatrix}
			\mathbf{I}_n & \nabla V \\ \mathbf{0} & \mathbf{I}_n
		\end{bmatrix}\pq,
		 \qquad \text{or}\qquad u\pq = \begin{bmatrix}
		\mathbf{I}_n & \mathbf{0} \\ \nabla V & \mathbf{I}_n
	\end{bmatrix}\pq,
\end{align*}
or a finite composition of these maps. And both of these maps are symplectic for any potential $V \in C^2(\R^{2n})$. Since symplectic maps form a group with composition as the group operation, LA-SympNets are symplectic by definition. This general form of SympNets (covering not only LA-SympNets, but G-SympNets as well) is visualized in \Cref{fig:SympNet}. In particular, we choose $V(\ve{x})= \frac{1}{2}\ve{x}^T \mathbf{S}_i \ve{x}, ~i=1,...,m$, for the linear modules and $V(\ve{x})= \ve{a}^T(\int \sigma)(\ve{x} + \ve{b})$, where $\int \sigma$ is an antiderivative of $\sigma$, for the activation modules.
\begin{figure}
	\centering
	\includegraphics[width=\textwidth]{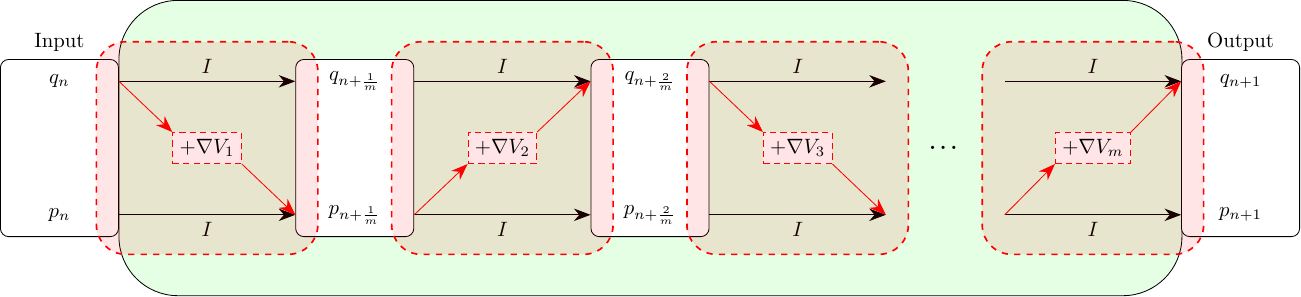}
	\caption{General idea of the SympNet architecture. The potentials $V_i$ are parametrized by different kinds of functions with trainable parameters.}
	\label{fig:SympNet}
\end{figure}
\cite{JanB25} also gives a universal approximation theorem for the LA-SympNets constructed in \Cref{def:SympNet}, i.e., LA-SympNets can approximate any symplectic map with arbitrary accuracy as long as we use sufficiently many layers. 
\subsection{Symplectic autoencoders}
In this section, we define the symplectic convolutional autoencoder (SympCAE) architecture. First, we define all the necessary modules in \Cref{subsubsec:conv,subsubsec:psd,subsubsec:pool} before building the autoencoder in \Cref{subsubsec:build}
\subsubsection{Convolutional modules}
\label{subsubsec:conv}
We utilize the SympNet architecture to construct symplectic convolutional layers via fixing some of the weights of the convolutional layers to be identity and some of them to be symmetric. We explain the idea with an example before we provide the formal definitions.\\
Consider the two input and output channel case, which can be represented as follows 
\begin{equation}\label{eqn:2b2}
	\begin{bmatrix}
		\ma{T}_{1,1}&\ma{T}_{1,2}\\
		\ma{T}_{2,1}&\ma{T}_{2,2}\\
	\end{bmatrix}
	\ve{x}=\ve{y}.
\end{equation}
Using \Cref{def:SympNet}, a linear symplectic convolution layer can be constructed by parameterizing the Toeplitz matrices in \cref{eqn:2b2} with $\ma{T}_{1,1}=\ma{T}_{2,2}=\ma{I}$, $\ma{T}_{2,1}=\ma{0}$, and $\ma{T}_{1,2}=\ma{T}_{1,2}^T$, resulting in an upper triangular symplectic layer. By splitting the input data into two channels corresponding to position and momenta in Hamiltonian dynamics, one can achieve a symplectic transformation by only parameterizing the weights of the convolutional layer. For instance, setting one channel to be an identity matrix can be easily accomplished by setting the weight tensor of length $l=3$ with $\te{w}_{1,1,\cdot}=[0,1,0]$. \begin{remark}
We note that directly implementing \eqref{eqn:2b2} by parameterizing the convolutional layers is not the most efficient method for implementing the symplectic convolutional layer. This is because three of the channels are fixed. Therefore, instead of considering the case with two inputs and two outputs, it is possible to implement an equivalent formulation by parameterizing only one input and one output channel.
\end{remark}
Nevertheless, many successful convolution autoencoder architectures depends on increasing or decreasing the number of channels. Next, we show a possible way to increase the number of input channels. Consider the following map
\begin{equation}\label{eqn:symplectic-copy}
	\ma{A}=
	\begin{bmatrix}
		c\ma{I}& \ma{0}\\
		c\ma{I}& \ma{0}\\
		\ma{0} & c\ma{I}\\
		\ma{0} & c\ma{I}
	\end{bmatrix}
\end{equation}
for some constant $c\in \R$. Using \Cref{eq:linearSymplecticity}, the linear map $\ma{A}$ is a symplectic lift if $2c^2=1$ is satisfied. Combining \cref{eqn:symplectic-copy} with a convolution layer of the same input and output channel number, we can obtain symplectic convolutional lifts. For example, considering the four input and output channel case, we can define a symplectic convolution layer as
\begin{equation}\label{eqn:4by4}
	\ma{B}=
	\begin{bmatrix}
		\ma{I} & \ma{0} & \ma{T}_1& \ma{T}_2\\
		\ma{0} & \ma{I} & \ma{T}_2& \ma{T}_3\\
		\ma{0} & \ma{0} & \ma{I}  & \ma{0}\\
		\ma{0} & \ma{0} & \ma{0} & \ma{I}
	\end{bmatrix},
\end{equation}
where the Toeplitz matrices $\ma{T}_{i}$ for $i=1,2,3$ are symmetric. By composing the symplectic layers of \cref{eqn:symplectic-copy} and \cref{eqn:4by4}, we obtain a symplectic convolutional lifting layer with two input channels and four output channels as follows:
\begin{equation}\label{eqn:4by2}
	\ma{C}=\ma{B}\ma{A}=
	\begin{bmatrix}
		c\ma{I}& \ma{T}_4\\
		c\ma{I}& \ma{T}_5\\
		\ma{0} & c\ma{I}\\
		\ma{0} & c\ma{I}
	\end{bmatrix}
\end{equation}
where $\ma{T}_4=c(\ma{T}_1+\ma{T}_2)$ and $\ma{T}_5=c(\ma{T}_2+\ma{T}_3)$, which are symmetric Toeplitz matrices.
\begin{remark}
	In practice, we do not need to construct the matrices \cref{eqn:symplectic-copy} and \cref{eqn:4by4}, we can parametrize convolutional layers using \cref{eqn:4by2}, to increase the channel number. Moreover, this parametrization can be generalized as long as the number of output channels $C_\text{out}$ is even and divisible by the number of input channels $C_\text{in}$, i.e., $C_\text{in} \mathrel{|} C_\text{out}$, as shown in \Cref{def:convmatrices}.
\end{remark}
To give a formal definition of the symplectic convolutional modules, we first need to introduce some sets of matrices.
\begin{defi}
	\label{def:convmatrices}
	Let $N, N_1, N_2 \in \N$.
	\begin{itemize}
		\item [$(i)$] The \emph{set of $N \times N$ Toeplitz matrices} $\T_N$ is given by
		\begin{align*}
			\T(N):=\left\{\ma{T}=\begin{bmatrix}
				t_{11} & \hdots & t_{1d}\\
				\vdots & \ddots & \vdots\\
				t_{d1} & \hdots & t_{dd}
			\end{bmatrix} \in \R^{N \times N}\strich \begin{aligned}
			&\ma{T} \text{ has form \cref{eqn:1dToep}, i.e., }t_{i,j}=t _{i+k,j+k},\\
			&i,j=1,...,N,\\
			&k=1,...,N-\max\{i,j\} 
			\end{aligned}\right\}.
		\end{align*}
		\item [$(ii)$] The \emph{set of symmetric $N \times N$ Toeplitz matrices} $\T_\text{sym}(N)$ is given by
		\begin{align*}
			\T_{\text{sym}}(N):=\left\{\ma{T} \in \T(N)\strich \ma{T}=\ma{T}^T\right\}.
		\end{align*}
		\item [$(iii)$] The \emph{set of block Toeplitz matrices} $\T(b, N)$ with block size $N \times N$ and $b \times b$ blocks is given by
		\begin{align*}
			\T(b, N):=\left\{\begin{bmatrix}
				\ma{T}_{1,1} & \hdots & \ma{T}_{1,b}\\
				\vdots & \ddots & \vdots\\
				\ma{T}_{b,1} & \hdots & \ma{T}_{b,b}
			\end{bmatrix}
			\in \R^{bN \times bN}\strich 
			\begin{aligned}
				&\ma{T}_{i,j} = \ma{T}_{i+k,j+k}\in \T(N),\\
				&i,j=1,...,b,\\
				&k=1,...,b-\max\{i,j\}
			\end{aligned}
			\right\}.
		\end{align*}
		\item[$(iv)$] The \emph{set of 1D block-symmetric Toeplitz matrices} $\T_\text{sym}^\text{1D}(b, N)$ with block size $N \times N$ and $b \times b$ blocks is given by
		\begin{align*}
			\T_\text{sym}^\text{1D}(b, N):=\left\{\begin{bmatrix}
				\ma{T}_{1,1} & \hdots & \ma{T}_{1,b}\\
				\vdots & \ddots & \vdots\\
				\ma{T}_{b,1} & \hdots & \ma{T}_{b,b}
			\end{bmatrix}
			\in \T(b,N)\strich 
			\begin{aligned}
				&\ma{T}_{ij}=\ma{T}_{ji} \in \T_\text{sym}(N),\\
				&i,j=1,...,b 
			\end{aligned}
			 \right\}.
		\end{align*}
		\item[$(v)$] The \emph{set of block-block Toeplitz matrices} $\T(b_1, b_2, N)$ with $b_1 \times b_1$ block Toeplitz matrices, which contain $b_2 \times b_2$ blocks of size $N \times N$ themselves, is given by
		\begin{align*}
			\T(b_1, b_2, N):=\left\{\begin{bmatrix}
				\ma{T}_{1,1} & \hdots & \ma{T}_{1,b_1}\\
				\vdots & \ddots & \vdots\\
				\ma{T}_{b_1,1} & \hdots & \ma{T}_{b_1,b_1}
			\end{bmatrix}
			\in \R^{b_1 b_2 N \times b_1 b_2 N}\strich 
			\begin{aligned}
				&\ma{T}_{i,j} = \ma{T}_{i+k,j+k}\in \T(b_2,N),\\
				&i,j=1,...,b_1,\\
				&k=1,...,b_1-\max\{i,j\}
			\end{aligned}
			\right\}.
		\end{align*}
		\item[$(vi)$] The \emph{set of 2D block-symmetric Toeplitz matrices} $\T_\text{sym}^\text{2D}(b, N_2, N_1)$ with block sizes $N_1, N_2$ and $b\times b$ blocks is given by
		\begin{align*}
			\T_\text{sym}^\text{2D}(b, N_2, N_1):=\left\{\begin{bmatrix}
				\ma{T}_{1,1} & \hdots & \ma{T}_{1,b}\\
				\vdots & \ddots & \vdots\\
				\ma{T}_{b,1} & \hdots & \ma{T}_{b,b}
			\end{bmatrix}
			\in \T(b,N_2, N_1)\strich \begin{aligned}
				& \ma{T}_{ij}=\ma{T}_{ji} \in \T_\text{sym}^{\text{1D}}(N_2, N_1),\\
				&i,j=1,...,b
			\end{aligned}\right\}.
		\end{align*}
		\item[$(vii)$] The \emph{set of 1D symplectic convolutional lifting matrices} with $C_{\text{in}}$ input channels and $C_{\text{out}}$ output channels $\C_\text{1D}(C_{\text{out}}, C_\text{in}, N)$ is given by
		\begin{align*}
			\C_{\text{1D}}(C_{\text{out}}, C_\text{in}, N):=&\left\{\begin{bmatrix}
				c \ma{I} & \ma{T}_1 \\ \vdots & \vdots \\ c \ma{I} & \ma{T}_d \\
				\ma{0} & c \ma{I} \\ \vdots & \vdots \\ \ma{0} & c \ma{I}
			\end{bmatrix}\in \R^{C_\text{out} N\times C_\text{in}N}\strich
			\begin{aligned}
				&\ma{T}_i \in \T_\text{sym}^{\text{1D}}(C_\text{in}/2, N),\\
				&c = \sqrt{\frac{1}{d}},\\
				&d=\frac{C_\text{out}}{C_\text{in}}
			\end{aligned}
			\right\}\\
			&\cup \left\{\begin{bmatrix}
				c \ma{I} & \ma{0} \\ \vdots & \vdots \\ c \ma{I} & \ma{0} \\
				\ma{T}_1 & c \ma{I} \\ \vdots & \vdots \\ \ma{T}_d & c \ma{I}
			\end{bmatrix}\in \R^{C_\text{out} N\times C_\text{in}N}\strich 
			\begin{aligned}
				&\ma{T}_i \in \T_\text{sym}^{\text{1D}}(C_\text{in}/2, N),\\
				&c = \sqrt{\frac{1}{d}},\\
				&d=\frac{C_\text{out}}{C_\text{in}}
			\end{aligned}
			\right\},
		\end{align*}
		where we assume $C_{\text{in}} \mathrel{|}C_{\text{out}}$ and $C_{\text{in}}$ is even.
		\item[$(viii)$] The \emph{set of 2D symplectic convolutional lifting matrices} with $C_{\text{in}}$ input channels and $C_{\text{out}}$ output channels $\C_\text{2D}(C_{\text{out}}, C_\text{in}, N_2, N_1)$ is given by
		\begin{align*}
			\C_{\text{2D}}(C_{\text{out}}, C_\text{in}, N_2, N_1):=&\left\{\begin{bmatrix}
				c \ma{I} & \ma{T}_1 \\ \vdots & \vdots \\ c \ma{I} & \ma{T}_d \\
				\ma{0} & c \ma{I} \\ \vdots & \vdots \\ \ma{0} & c \ma{I}
			\end{bmatrix}\in \R^{C_\text{out} N_1 N_2\times C_\text{in}N_1 N_2}\strich
			\begin{aligned}
				&\ma{T}_i \in \T_\text{sym}^{\text{2D}}(C_\text{in}/2, N_2, N_1),\\
				&c = \sqrt{\frac{1}{d}},\\
				&d=\frac{C_\text{out}}{C_\text{in}}
			\end{aligned}
			\right\}\\
			&\cup \left\{\begin{bmatrix}
				c \ma{I} & \ma{0} \\ \vdots & \vdots \\ c \ma{I} & \ma{0} \\
				\ma{T}_1 & c \ma{I} \\ \vdots & \vdots \\ \ma{T}_d & c \ma{I}
			\end{bmatrix}\in \R^{C_\text{out} N_1 N_2 \times C_\text{in}N_1 N_2}\strich 
			\begin{aligned}
				&\ma{T}_i \in \T_\text{sym}^{\text{2D}}(C_\text{in}/2, N_2, N_1),\\
				&c = \sqrt{\frac{1}{d}},\\
				&d=\frac{C_\text{out}}{C_\text{in}}
			\end{aligned}
			\right\},
		\end{align*}
		where we assume $C_{\text{in}} \mathrel{|}C_{\text{out}}$ and $C_{\text{in}}$ is even.
		\item[$(ix)$] The \emph{set of 1D symplectic convolutional projection matrices} with $C_{\text{in}}$ input channels and $C_{\text{out}}$ output channels $\C_\text{1D}^T(C_{\text{out}}, C_\text{in}, N)$ is given by
		\begin{align*}
			&\C_{\text{1D}}^T(C_{\text{out}}, C_\text{in}, N):=\\
			&\left\{\begin{bmatrix}
				c \ma{I} & \hdots & c \ma{I} & \ma{T}_1 & \hdots & \ma{T}_d\\
				\ma{0} & \hdots &\ma{0} & c \ma{I} & \hdots & c \ma{I}
			\end{bmatrix}\in \R^{C_\text{out} N\times C_\text{in}N}\strich
			\begin{aligned}
				&\ma{T}_i \in \T_\text{sym}^{\text{1D}}(C_\text{out}/2, N),\\
				&c = \sqrt{\frac{1}{d}},\\
				&d=\frac{C_\text{out}}{C_\text{in}}
			\end{aligned}
			\right\}\\
			&\cup \left\{\begin{bmatrix}
				c \ma{I} & \hdots & c \ma{I} & \ma{0} & \hdots &\ma{0}\\
				\ma{T}_1 & \hdots & \ma{T}_d & c \ma{I} & \hdots & c \ma{I}
			\end{bmatrix}\in \R^{C_\text{out} N\times C_\text{in}N}\strich 
			\begin{aligned}
				&\ma{T}_i \in \T_\text{sym}^{\text{1D}}(C_\text{out}/2, N),\\
				&c = \sqrt{\frac{1}{d}},\\
				&d=\frac{C_\text{in}}{C_\text{out}}
			\end{aligned}
			\right\},
		\end{align*}
		where we assume $C_{\text{out}} \mathrel{|}C_{\text{in}}$ and $C_{\text{in}}$ is even.
		\item[$(x)$] The \emph{set of 2D symplectic convolutional projection matrices} with $C_{\text{in}}$ input channels and $C_{\text{out}}$ output channels $\C_\text{2D}^T(C_{\text{out}}, C_\text{in}, N_2, N_1)$ is given by
		\begin{align*}
			&\C_{\text{2D}}^T(C_{\text{out}}, C_\text{in}, N_2, N_1):=\\
			&\left\{\begin{bmatrix}
				c \ma{I} & \hdots & c \ma{I} & \ma{T}_1 & \hdots & \ma{T}_d\\
				\ma{0} & \hdots &\ma{0} & c \ma{I} & \hdots & c \ma{I}
			\end{bmatrix}\in \R^{C_\text{out} N\times C_\text{in}N}\strich
			\begin{aligned}
				&\ma{T}_i \in \T_\text{sym}^{\text{2D}}(C_\text{out}/2, N_2, N_1),\\
				&c = \sqrt{\frac{1}{d}},\\
				&d=\frac{C_\text{out}}{C_\text{in}}
			\end{aligned}
			\right\}\\
			&\cup \left\{\begin{bmatrix}
				c \ma{I} & \hdots & c \ma{I} & \ma{0} & \hdots &\ma{0}\\
				\ma{T}_1 & \hdots & \ma{T}_d & c \ma{I} & \hdots & c \ma{I}
			\end{bmatrix}\in \R^{C_\text{out} N\times C_\text{in}N}\strich 
			\begin{aligned}
				&\ma{T}_i \in \T_\text{sym}^{\text{2D}}(C_\text{out}/2, N_2, N_1),\\
				&c = \sqrt{\frac{1}{d}},\\
				&d=\frac{C_\text{in}}{C_\text{out}}
			\end{aligned}
			\right\},
		\end{align*}
		where we assume $C_{\text{out}} \mathrel{|}C_{\text{in}}$ and $C_{\text{in}}$ is even.
	\end{itemize}
\end{defi}
\begin{remark}
	Not that it is not necessary to choose all the Toeplitz matrices in \Cref{def:convmatrices} $(iv)$ to be symmetric for the resulting Toeplitz matrix to be symmetric. Hence a symplectic convolutional lifting matrix would still be symplectic if we chose the non-diagonal blocks of the block-symmetric Toeplitz matrices to not be symmetric. We still choose them to be symmetric since it lowers the number of trainable parameters per matrix and it mimics the structure of a 2D convolution \cref{eqn:2dToep}. The same holds true for the construction of the block-symmetric Toeplitz matrices for the 2D case.
\end{remark}
With these preparations we can formally define the convolutional modules.
\begin{defi}[Convolutional modules]
	\label{def:convmods}
	Let $N, N_1, N_2 \in \N$. The \emph{sets of 1D and 2D symplectic convolutional lifting modules} $\mathcal{M}_\text{Conv}^\text{1D}, \mathcal{M}_\text{Conv}^\text{2D}$ with $C_\text{in}$ (even) input channels and $C_\text{out}$ output channels are given by
	\begin{align*}
		\mathcal{M}_\text{Conv}^\text{1D}:=&\left\{x \mapsto \ma{A}_k \cdots \ma{A}_1 x \strich  \ma{A}_i \in \C_\text{1D}(C_{i-1},C_i, N),~ C_0=C_\text{in},~C_k=C_\text{out},\right.\\
		& \qquad \qquad \qquad \qquad\left. C_{i-1} \mathrel{|} C_{i},~i=1,...,k,~k \in \N\right\}.\\
		\mathcal{M}_\text{Conv}^\text{2D}:=&\left\{x \mapsto \ma{A}_k \cdots \ma{A}_1 x \strich  \ma{A}_i \in \C_\text{2D}(C_{i-1},C_i, N_2, N_1),~ C_0=C_\text{in},~C_k=C_\text{out},\right.\\
		& \qquad \qquad \qquad \qquad\left. C_{i-1} \mathrel{|} C_{i},~i=1,...,k,~k \in \N\right\}.
	\end{align*}
\end{defi}
It is easy to observe that the set of symplectic convolutional modules only contains symplectic maps, as we show in the following proposition.
\begin{prop}
	\label{prop:conv}
	Every symplectic convolutional module is symplectic.
\end{prop}
\begin{proof}
	Without loss of generality, we just consider the 1D case. First, we show that $\ma{A}_i \in \C_{\text{1D}}(C_{i-1},C_i, N)$ for $i=1,...,k$ is symplectic by observing
	\begin{align*}
		\ma{A}_i^T \ma{J}_{C_i N} \ma{A}_i &=
		\begin{bmatrix}
			c \ma{I} & \hdots & c \ma{I} & \ma{0} &\hdots & \ma{0}\\
			\ma{T}_1 & \hdots &\ma{T}_d & c \ma{I} & \hdots & c \ma{I}
		\end{bmatrix}\begin{bmatrix}
			& & & \ma{I} & & \\
			& & & & \ddots & \\
			& & & & & \ma{I} \\
			-\ma{I} & & & & & \\
			& \ddots & & & & \\
			& & -\ma{I} & & & \\
		\end{bmatrix}
		\begin{bmatrix}
			c \ma{I} & \ma{T}_1 \\ \vdots & \vdots \\ c \ma{I} & \ma{T}_d \\
			\ma{0} & c \ma{I} \\ \vdots & \vdots \\ \ma{0} & c \ma{I}
		\end{bmatrix}\\
		&=\begin{bmatrix}
			\ma{0} & d c^2 \ma{I}\\
			-dc^2 \ma{I} & \sum_{i=1}^{d}c\ma{T}_i - \sum_{i=1}^{d} c\ma{T}_i
		\end{bmatrix}= \ma{J}_{C_{i-1}N},
	\end{align*}
	because $c = \sqrt{1/d}$. Now the statement of the proposition follows from
	\begin{align*}
		(\ma{A}_k \cdots \ma{A}_1)^T \ma{J}_{C_\text{out} N} \ma{A}_k \cdots \ma{A}_1
		&= \ma{A}_1^T \cdots \ma{A}_k^T\ma{J}_{C_\text{out} N} \ma{A}_k \cdots \ma{A}_1\\
		&= \ma{A}_1^T \cdots \ma{A}_{k-1}^T\ma{J}_{C_{k-1} N} \ma{A}_{k-1} \cdots \ma{A}_1\\
		&= \ma{A}_1^T \ma{J}_{C_1 N}\ma{A}_1 = \ma{J}_{C_\text{in}N}.
	\end{align*}
\end{proof}
For the decoder we also need convolutional projection modules. They are used the same way as \texttt{ConvTranspose} module in a classical convolutional autoencoder. Hence, we use "convT" in our notation
\begin{defi}
	\label{def:convTmods}
	Let $N, N_1, N_2 \in \N$. The \emph{sets of 1D and 2D symplectic convolutional projection modules} $\mathcal{M}_\text{ConvT}^\text{1D}, \mathcal{M}_\text{ConvT}^\text{2D}$ with $C_\text{in}$ input channels and $C_\text{out}$ (even) output channels are given by
	\begin{align*}
		\mathcal{M}_\text{ConvT}^\text{1D}:=&\left\{x \mapsto \ma{A}_k \cdots \ma{A}_1 x \strich  \ma{A}_i \in \C_\text{1D}^T(C_{i-1},C_i, N),~ C_0=C_\text{in},~C_k=C_\text{out},\right.\\
		& \qquad \qquad \qquad \qquad\left. C_{i} \mathrel{|} C_{i-1},~i=1,...,k,~k \in \N\right\}.\\
		\mathcal{M}_\text{ConvT}^\text{2D}:=&\left\{x \mapsto \ma{A}_k \cdots \ma{A}_1 x \strich  \ma{A}_i \in \C_\text{2D}^T(C_{i-1},C_i, N_2, N_1),~ C_0=C_\text{in},~C_k=C_\text{out},\right.\\
		& \qquad \qquad \qquad \qquad\left. C_{i} \mathrel{|} C_{i-1},~i=1,...,k,~k \in \N\right\}.
	\end{align*}
\end{defi}
Again, it is straightforward to prove that the convolutional projection modules are indeed symplectic.
\begin{prop}
	Every symplectic convolutional projection module is symplectic in the sense that it is the symplectic inverse of a symplectic lifting.
\end{prop}
\begin{proof}
	Follows from \Cref{prop:conv} and the fact that $A_i^T \in \C_\text{1D}^T(C_{i},C_{i-1}, N)$ or\\
	$A_i^T \in \C_\text{2D}^T(C_{i},C_{i-1}, N_2, N_1)$.
\end{proof}
\subsubsection{PSD-like layers}
\label{subsubsec:psd}
Proper symplectic decomposition (PSD) \cite{PenMoh16} is an important and well established method in symplectic model order reduction. We want to make use of PSDs model order reduction capabilities by including PSD-like layers defined in \Cref{def:PSDmods} into our autoencoder architecture like it was already done in \cite{BraKra23}. In particular, we consider the following set of the symplectic matrices,
\begin{equation*}
	\mathbb{M}(2n,2k):=Sp(2k,\R^{n}) \cap
	\left\{
	\begin{bmatrix}
		\ma{\Phi}&\ma{0}\\
		\ma{0}&\ma{\Phi}
	\end{bmatrix} \strich \ma{\Phi} \in \R^{n\times k}
	\right\} = \left\{\begin{bmatrix}
		\ma{\Phi}&\ma{0}\\
		\ma{0}&\ma{\Phi}
	\end{bmatrix}\strich \ma{\Phi}^T \ma{\Phi}= \ma{I}_k,\ma{\Phi} \in \R^{n\times k}\right\}.
\end{equation*}
\begin{defi}[PSD-like modules]
	\label{def:PSDmods}
	Let $n, k \in \N,~n\geq k$, then the \emph{set of PSD-like modules} is given by
	\begin{align}
		\label{eqn:PSDmods}
		\begin{aligned}
			\mathcal{M}_\text{\text{PSD}}:=&\left\{x \mapsto A^+ x \strich A \in \mathbb{M}(2n,2k)\right\}\\
			=&\left\{x \mapsto \begin{bmatrix}
				\ma{\Psi} & \ma{0}\\ \ma{0} & \ma{\Psi}
			\end{bmatrix} x \strich \ma{\Psi}\ma{\Psi}^T = \ma{I}_k,~ \ma{\Psi}\in \R^{k \times n}\right\}.
		\end{aligned}
	\end{align}
	We also define the \emph{set of PSD-like transpose modules} as
	\begin{align}
		\label{eqn:PSDTmods}
		\begin{aligned}
			\mathcal{M}_\text{\text{PSDT}}:=&\left\{x \mapsto A x \strich A \in \mathbb{M}(2n,2k)\right\}\\
			=&\left\{x \mapsto \begin{bmatrix}
				\ma{\Psi} & \ma{0}\\ \ma{0} & \ma{\Psi}
			\end{bmatrix} x \strich \ma{\Psi}^T\ma{\Psi} = \ma{I}_k,~ \ma{\Psi}\in \R^{n \times k}\right\}.
		\end{aligned}
	\end{align}
\end{defi}
\begin{prop}
	The PSD-like modules are symplectic in the sense that they are the symplectic inverse of a symplectic lifting. Also, the symplectic transpose modules are symplectic liftings.
\end{prop}
\begin{proof}
	Holds by definition of the PSD-like modules \cref{eqn:PSDmods} and PSD-like transpose modules \cref{eqn:PSDTmods}.
\end{proof}
\subsubsection{Symplectic pooling}
\label{subsubsec:pool}
To construct a general nonlinear symplectic encoder, we lastly construct symplectic max-pooling layers by utilizing PSD-like matrices. First, let us define the equivalent matrix form of \texttt{max-pooling} for equal stride and kernel size. For simplicity, we consider the following one-channel 1D input signal as an example, to derive the general definition:
\begin{equation*}
	\ve{x}=\begin{bmatrix}
		2\\
		1\\
		3\\
		5\\
	\end{bmatrix}.
\end{equation*}
A 1D \texttt{max-pooling} operation $\pool$ with a stride and kernel size of two, padding of zero and dilation of one can be equivalently written with the following matrix:
\begin{equation*}
	\ma{\Phi}(\ve{x})=
	\begin{bmatrix}
		1&0&0&0\\
		0&0&0&1
	\end{bmatrix},
\end{equation*} 
such that $\pool(\ve{x})=\ma{\Phi}(\ve{x})\ve{x}$, which means that $\ma{\Phi}(\ve{x})$ is the Jacobian of the \texttt{max-pooling} operation at $\ve{x}$. Note that $\ma{\Phi}(\ve{x})\ma{\Phi}(\ve{x})^T=\ma{I}_2$. We will show later that this is no coincidence. 
Hence, assuming we have an input signal with two input channels, applying the same pooling operation to both channels can equivalently be expressed as:
\begin{equation*}
	\ma{P}(\ve{x})=
	\begin{bmatrix}
		\ma{\Phi}(\ve{x})&\ma{0}\\
		\ma{0}&\ma{\Phi}(\ve{x})
	\end{bmatrix},
\end{equation*}  
where $\ma{P}(\ve{x})$ is the symplectic inverse \eqref{eq:symplecticInverse} of $\ma{P}(\ve{x})^T\in \mathbb{M}(2N,2k)$, which is a symplectic projection.
\begin{defi}[Symplectic pooling]
	\label{def:pool}
	Let $\pool$ be the \texttt{max-pooling} operation with dilation of one, padding of zero and stride and kernel size $k\in \N$. Furthermore, let $\ma{\Phi}(\ve{x})\in \R^{N/k \times N}$ be the Jacobian of $\pool$ at $\ve{x} \in \R^{N},$ where $k\mathrel{|}N$. The \emph{symplectic pooling modules} for two input channels are given by
	\begin{align*}
		p_\text{up}\begin{bmatrix}
			\ve{x}_1 \\ \ve{x}_2
		\end{bmatrix}:= \begin{bmatrix}
			\ma{\Phi}(\ve{x}_1) \ve{x}_1\\
			\ma{\Phi}(\ve{x}_1) \ve{x}_2
		\end{bmatrix}, \qquad
		p_\text{low}\begin{bmatrix}
			\ve{x}_1 \\ \ve{x}_2
		\end{bmatrix}:= \begin{bmatrix}
			\ma{\Phi}(\ve{x}_2) \ve{x}_1\\
			\ma{\Phi}(\ve{x}_2) \ve{x}_2
		\end{bmatrix}.
	\end{align*}
	We denote the \emph{set of symplectic pooling modules} by
	\begin{align*}
		\mathcal{M}_\text{P}:=\left\{p\strich p \text{ is a symplectic pooling module}\right\}.
	\end{align*}
\end{defi}
\begin{remark}
	It is sufficient to only consider the two input channel case, because we will flatten the input such that there is only one channel for the generalized momenta $\ve{p}$ and one for the generalized positions $\ve{q}$, which is what is important for symplecticity.
\end{remark}
\begin{prop}
	\label{prop:pool}
	The symplectic pooling is symplectic in the sense that it is the symplectic inverse of a symplectic lifting.
\end{prop}
\begin{proof}
	In the general case, fixing the stride and kernel size to be the same yields a mathematically equivalent form $\ma{\Phi}_{i,:}(\ve{x})=\ve{e}_{\pi_\ve{x}(i)}^T$ for any $\ve{x}\in \R^N$, where $\ve{e}_{\pi_\ve{x}(i)}$ is the $\pi_\ve{x}(i)$-th standard basis vector for some $i \in \mathbb{R}$, for some $\pi_\ve{x} : \{1,...,N/k\} \to \{1,...,N\}$ depending on the maximum element in the window spanned by the pooling kernel, and $\ma{\Phi}_{i,:}(\ve{x})$ denoting the $i$-th row of $\ma{\Phi}(\ve{x})$. Also, kernel size and stride being the same, results in every component of $\ve{x}$ only appearing in one window. Hence, $\pi_\ve{x}$ is injective. With this observation, it holds
	\begin{align*}
		\left(\ma{\Phi}(\ve{x})\ma{\Phi}(\ve{x})^T\right)_{ij}=\ma{\Phi}(\ve{x})_{i,:}\left(\ma{\Phi}(\ve{x})^T\right)_{:,j}=\ve{e}_{\pi_\ve{x}(i)}^T \ve{e}_{\pi_\ve{x}(j)}= \delta_{\pi_\ve{x}(i)\pi_\ve{x}(j)}=\delta_{ij}
	\end{align*}
	for $i,j=1,...,N/k$, which yields $\ma{\Phi}(\ve{x})\ma{\Phi}(\ve{x})^T = \ma{I}_{N/k}$. According to \cref{eqn:PSDmods} this is sufficient to show that the Jacobians of $p_\text{up}$ and $p_\text{low}$ are symplectic inverses of symplectic liftings. Hence, this also holds for $p_\text{up}$ and $p_\text{low}$ themselves.
\end{proof}
Following our example, we want to introduce symplectic unpooling $\unpool$ as well
\begin{align*}
	\unpool(\pool(\ve{x}))=\begin{bmatrix}
		2 \\ 0 \\ 0 \\ 5
	\end{bmatrix}.
\end{align*}
This means that we can represent the unpooling operation by $\ma{\Phi}(\ve{x})^T$, i.e.,
\begin{align*}
	 \unpool(\pool(\ve{x}))=\ma{\Phi}(\ve{x})^T\ma{\Phi}(\ve{x})\ve{x}.
\end{align*}
\begin{defi}[Symplectic unpooling]
	Let $\pool$ be the \texttt{max-pooling} operation with dilation of one, padding of zero and stride and kernel size $k\in \N$. Furthermore, let $\ma{\Phi}(\ve{x})\in \R^{N/k \times N}$ be the Jacobian of $\pool$ at $\ve{x} \in \R^{N},$ where $k\mathrel{|}N$. Additionally, let $\ve{\tilde{x}}_1, \ve{\tilde{x}}_2\in \R^k$. The \emph{symplectic unpooling modules} for two input channels with respect to the previous pooling respresented by $\ma{\Phi}(\ve{x})$ are given by
	\begin{align*}
		u\begin{bmatrix}
			\ve{\tilde{x}}_1 \\
			\ve{\tilde{x}}_2
		\end{bmatrix}:=\begin{bmatrix}
				\ma{\Phi}(\ve{x})^T \ma{\tilde{x}}_1\\
				\ma{\Phi}(\ve{x})^T \ma{\tilde{x}}_2
		\end{bmatrix}.
	\end{align*}
	We denote the \emph{set of symplectic unpooling modules} by
	\begin{align*}
		\mathcal{M}_\text{U}:=\left\{u \strich u \text{ is a symplectic unpooling module}\right\}.
	\end{align*}
\end{defi}
\begin{prop}
	The symplectic unpooling modules are symplectic liftings.
\end{prop}
\begin{proof}
	The proof of \Cref{prop:pool} yields $\ma{\Phi}(\ve{x})\ma{\Phi}(\ve{x})^T = \ma{I}_k$, which is sufficient to show that the unpooling modules are symplectic according to \cref{eqn:PSDmods}.
\end{proof}
\subsubsection{Building the autoencoder}
\label{subsubsec:build}
Finally, we can define the symplectic encoder. The idea is to lift the dynamics first using convolutional and activation modules, before reducing them with a pooling and  a PSD-like layer. This approach is visualized in \Cref{fig:CNN}. We formally introduce the symplectic encoder with the following definition.
\begin{defi}[Symplectic encoder]
	\label{def:enc}
	The \emph{set of symplectic encoders} $\Psi_\text{Enc}$ is given by
	\begin{align*}
		\Psi_{\text{Enc}}:=\left\{\psi= z \circ p \circ w_l \circ y_l \circ \ldots \circ w_1 \circ y_1 \strich \begin{aligned}
			&y_i\in \mathcal{M}_\text{Conv},\ w_i \in \mathcal{M}_\text{A},\ z \in \mathcal{M}_\text{PSD},\\
			&p \in \mathcal{M}_\text{P},~
			i=1,...,l,~  l \in \N
		\end{aligned}\right\}.
	\end{align*}
\end{defi}
\begin{remark}
	Note that the dimensions and channel numbers of the outputs of the previous modules have to match the input dimension and channel of the next module. Since we assume the number of input channels to be 2 for activation, pooling and PSD-like modules, we flatten the input before applying them.
\end{remark}
In practice, one can use the symplectic inverse of each layer to construct a symplectic decoder, which would reduce the number of parameters. Nevertheless, we construct the symplectic decoder via \emph{symplectic inverse-like maps} to obtain a Petrov-Galerkin approach like in \cite{buchfink2023symplectic,BraKra23}.
Using the modules defined in the previous sections, we can define a general symplectic decoder by inverting the order of modules compared to the encoder.
\begin{defi}[Symplectic decoder]
	\label{def:decoder}
		The \emph{set of symplectic decoders} $\Psi_\text{Dec}$ is given by
	\begin{align*}
		\Psi_{\text{Dec}}:=\left\{\psi=  y_1 \circ w_1 \circ \ldots \circ w_l \circ y_l \circ u \circ z  \strich \begin{aligned}
			 &y_i\in \mathcal{M}_\text{ConvT},\ w_i \in \mathcal{M}_\text{A},\ z \in \mathcal{M}_\text{PSDT},\\
			 &u \in \mathcal{M}_\text{U},~i=1,...,l,~ l \in \N 
		\end{aligned}\right\}.
	\end{align*}
\end{defi}
Combining the symplectic encoder and decoder, we define the symplectic convolutional autoencoder as a pair of parametric mappings.
\begin{defi}[Symplectic autoencoder]
	The \emph{set of symplectic autoendocders} is given by
	\begin{align*}
		\Psi_\text{AE}:=\left\{\psi = \psi_\text{Dec} \circ \psi_\text{Enc} \strich \psi_\text{Dec} \in \Psi_{\text{Dec}},~\psi_\text{Enc} \in \Psi_{\text{Enc}}\right\}.
	\end{align*}
\end{defi}
\section{Numerical Results}
\label{sec:num}
In this section, we evaluate the performance of the symplectic convolutional autoencoder (SympCAE) on three different test cases: the linear wave equation, the nonlinear Schrödinger (NLS) equation, and the sine-Gordon (SG) equation. We compare the proposed autoencoder with the PSD autoencoder to assess accuracy. After learning a suitable embedding, we use SympNets~\cite{JinZZ20} to extrapolate the dynamics over time. We test the 1D SympCAE with the first two test cases, which are one-dimensional PDEs. Finally, we test the 2D SympCAE using the two-dimensional sine-Gordon equation.

To test the accuracy of the autoencoders, we use the following relative Frobenius error:	
\begin{equation}\label{eqn:rel_err}
	\bm{\varepsilon} = \dfrac{\|\ma{X} - \ma{Y}\|_F}{\|\ma{X}\|_F},
\end{equation}
where $\ma{X}$ consists of the trajectories of the fully-discrete ground truth model and $\ma{Y}$ is the approximation obtained via the autoencoder. Moreover, we examine the accuracy of the SympNet over the time domain using the following relative error:
\begin{equation}\label{eqn:rel_time_err}
	\bm{\varepsilon}(t_i) = \dfrac{\|\ve{u}(t_i) - \ve{\tilde{u}}(t_i)\|_2}{\|\ve{u}(t_i)\|_2}, \quad i=1,\ldots,N_t,
\end{equation}
where $N_t$ denotes number of time steps.

\subsection{Wave equation}
\label{subsec:wave}
Following~\cite{sharma2025nonlinear}, we first consider the one-dimensional linear wave equation of the form:
\begin{equation}\label{eqn:wave}
	\begin{aligned}
		&u_{tt}(x,t) =cu_{xx}(x,t),\\
		& u(x,0) = u^0(x),\\
		& u_t(x,0) = u_t^0(x), & x \in \Omega,    
	\end{aligned}
\end{equation}
where $c$ denotes the transport velocity. We set boundary conditions to be periodic. The wave equation~\eqref{eqn:wave} is a simple example of a Hamiltonian systems. To derive the Hamiltonian form of the wave equation~\eqref{eqn:wave}, let us define the variables $p=u_t$ and $q=u$, which yields the Hamiltonian form
\begin{equation}\label{eqn:Ham_wave}
	\frac{\partial z}{\partial t}
	=
	\begin{bmatrix}
		0&1\\-1&0
	\end{bmatrix}
	\frac{\delta H}{\delta z}
	,\quad z=
	\begin{bmatrix}
		q\\p
	\end{bmatrix},
\end{equation}
where $\delta $ denotes the variational derivative and the Hamiltonian is given as
\begin{align*}
	\mathcal{H}(z) = \dfrac{1}{2}\int_{\Omega}
	cq_x^2+p^2\,dx.
\end{align*}
Following~\cite{PenMoh16}, we discretize the space using a structure-preserving finite difference approach with $N$ equidistant grid points, which leads to the following semi-discrete Hamiltonian:
\begin{equation}\label{eqn:wave-disc-Ham}
	H=\sum_{i=1}^{N}\Delta x \left[\frac{1}{2}\ve{p}_i^2+\frac{c(\ve{q}_{i+1}-\ve{q}_{i})^2}{2\Delta x^2}+\frac{c(\ve{q}_{i}-\ve{q}_{i-1})^2}{2\Delta x^2}\right],
\end{equation}
where $\ve{p}_i=u_t(x_i,t), \ve{q}_i=u(t,x_i)$, and $x_i=i \Delta x$. The semi-discrete Hamiltonian form of the wave equation~\eqref{eqn:Ham_wave} is expressed as follows:
\begin{equation}\label{eqn:Wave-Ham-ODE}
	\frac{d  \ve{z}}{d t}= \ma{K}  \ve{z}, 
\end{equation}
where
$$ 
\ve{z}=\begin{bmatrix}
	\ve{q}\\  \ve{p}
\end{bmatrix}, \quad \ma{K}=\begin{bmatrix}
	\ma{0}_N&\ma{I}_N\\ c\ma{D}_{xx}& \ma{0}_N
\end{bmatrix}, 
$$
$\ma{I}_N\in \mathbb{R}^{N\times N}$ is the identity matrix, $\ma{0}_N\in \mathbb{R}^{N\times N}$ is a matrix of zeros, $\ma{D}_{xx}\in \mathbb{R}^{N\times N}$ is the central difference approximation of $\partial_{xx}$, and $\ve{q}, \ve{p} \in \mathbb{R}^{N}$ are the discretized variables $q,p$. 
To preserve the symplectic structure of the wave equation~\eqref{eqn:Wave-Ham-ODE}, we discretize it using the symplectic Euler method~\cite{hairer2006structure}:
\begin{equation}
	\begin{aligned}
		\ve{q}^{n+1} &= \ve{q}^n + \Delta t \cdot \frac{\partial H}{\partial p}(\ve{q}^n, \ve{p}^{n+1}),\\
		\ve{p}^{n+1} &= \ve{p}^n - \Delta t \cdot \frac{\partial H}{\partial q}(\ve{q}^n, \ve{p}^{n+1}),
	\end{aligned}
\end{equation}
where the superscript $n$ denotes the time step $n$.
This yields the explicit scheme:
\begin{equation}\label{eqn:wave-fully-discrete}
	\ve{z}^{n+1} = 
	\begin{bmatrix} 
		\ma{I}_N & \Delta t \ma{I}_N \\ \ma{0}_N & \ma{I}_N  
	\end{bmatrix}
	\begin{bmatrix} 
		\ma{I}_N & \ma{0}_N \\
		c \Delta t \ma{D}_{xx} & \ma{I}_N 
	\end{bmatrix} \ve{z}^{n}.
\end{equation}
Notice that \eqref{eqn:wave-fully-discrete} has a symplectic structure that can be modeled through a symplectic convolutional module $y \in \mathcal{M}_\text{Conv}$. Therefore, instead of using \eqref{eqn:wave-fully-discrete}, we parameterize the \texttt{PyTorch} 1D CNN modules, resulting in a composition of two symplectic convolution layers. We set the spatial domain $\Omega = [0, 5]$, the transport velocity $c = 1$, and the number of grid points to $N = 1024$, which yields a discretized state $\ve{z} \in \mathbb{R}^{2048}$. To construct the training set, we simulate the system until time $t = 5$ with the initial conditions $u^0(x) = \exp(-(x - 2.5)^2)$, $u_t^0(x) = 0$, and $N_t = 1024$ time steps.

First, we train the SympCAE with the reshaped input data  $\te{X}\in \mathbb{R}^{N_t\times2\times N}$ to construct a symplectic autoencoder. Then, using the same data for training, we compare the reconstruction error~\eqref{eqn:rel_err} of the SympCAE with the PSD~\cite{PenMoh16} autoencoder for latent dimensions $r=1,2,3$ in \Cref{tab:wave}, which shows that the symplectic CNN outperforms PSD autoencoder. Moreover, we show the reconstructed states and corresponding absolute errors in \Cref{fig:wave-u,fig:wave-v} for latent dimension $r=1$, which similarly demonstrates that the SympCAE can learn the solution with good accuracy even with a very small latent dimension.

\begin{table}[tb]
	\caption{Linear wave equation: The table shows the performance of the autoencoder obtained using PSD and the SympCAE in capturing the dynamics of the ground truth model in terms of relative reconstruction error \eqref{eqn:rel_err} at latent dimensions $r = 1, 2, 3$. The best result for each latent dimension is highlighted in bold.}
	\label{tab:wave}
	\centering
	\renewcommand{\arraystretch}{1.2}
	\begin{tabular}{lll}
		$r$  &  $\bm{\varepsilon}_{\text{PSD}}$ & $\bm{\varepsilon}_{\text{SympCAE}}$ \\ \hline
		$1$ & $7.28\cdot 10^{-1}$  & $\mathbf{1.47\cdot 10^{-2}}$   \\ 
		$2$ & $3.60\cdot 10^{-1}$  & $\mathbf{1.21\cdot 10^{-2}}$    \\ 
		$3$ & $7.20\cdot 10^{-2}$  & $\mathbf{9.25\cdot 10^{-3}}$  \\ \hline 
	\end{tabular}
\end{table}

\begin{figure}[tb]
	\centering
	\begin{subfigure}{0.328\textwidth}
		\includegraphics[width=1\linewidth]{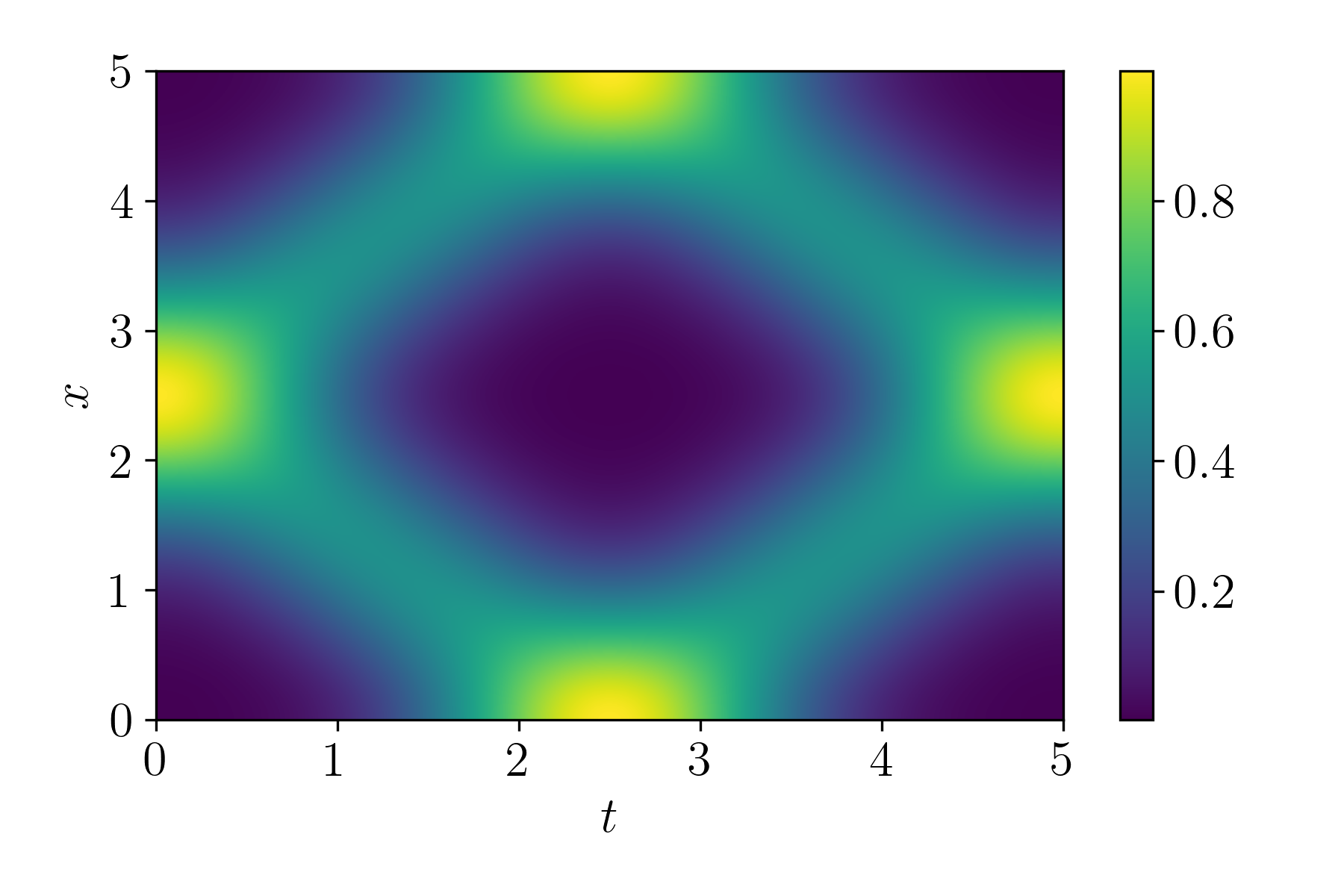}
		\caption{Ground truth}	
	\end{subfigure}
	\begin{subfigure}{0.328\textwidth}
		\includegraphics[width=1\linewidth]{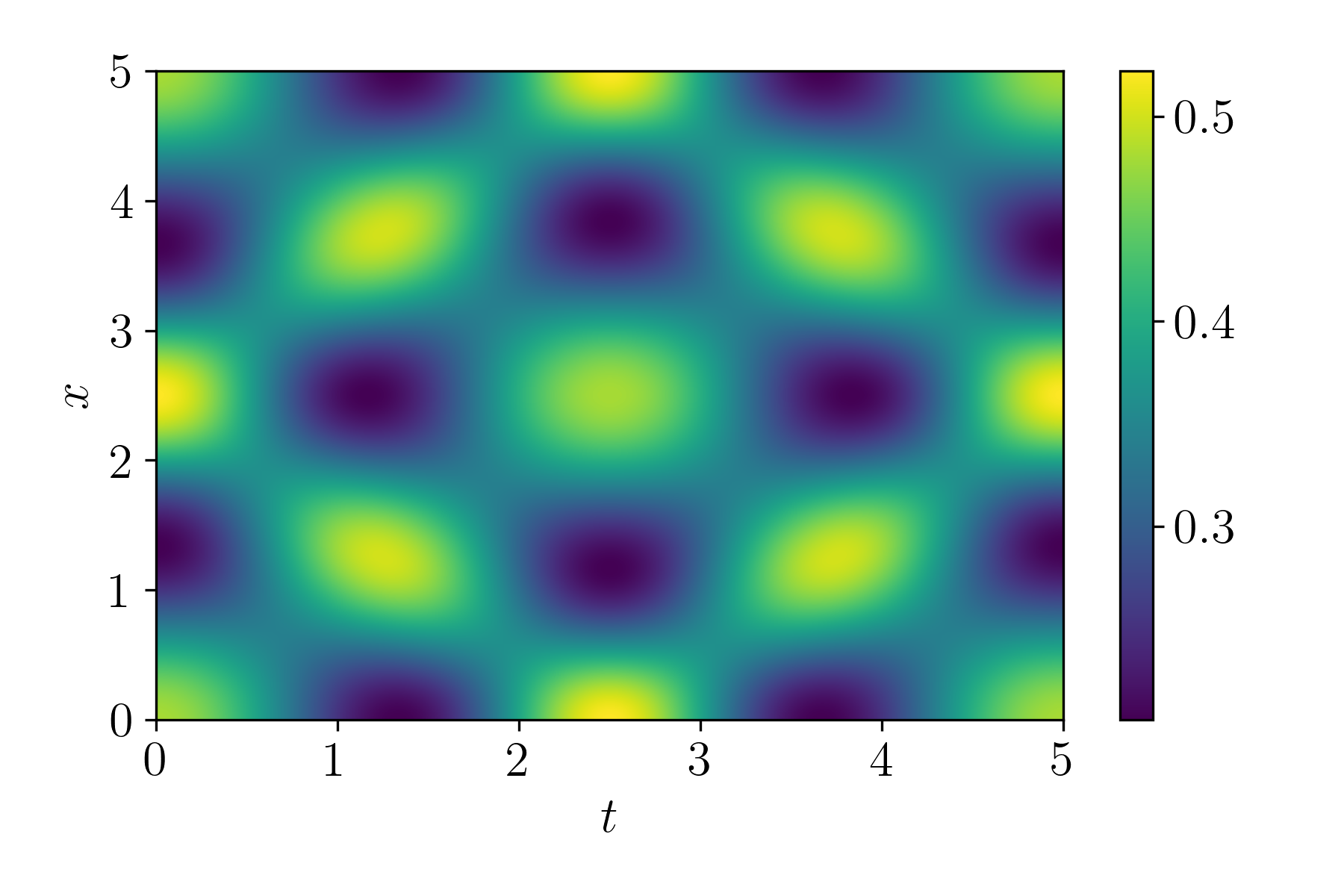} 
		\caption{Absolute PSD error }
	\end{subfigure}
	\begin{subfigure}{0.328\textwidth}
		\includegraphics[width=1\linewidth]{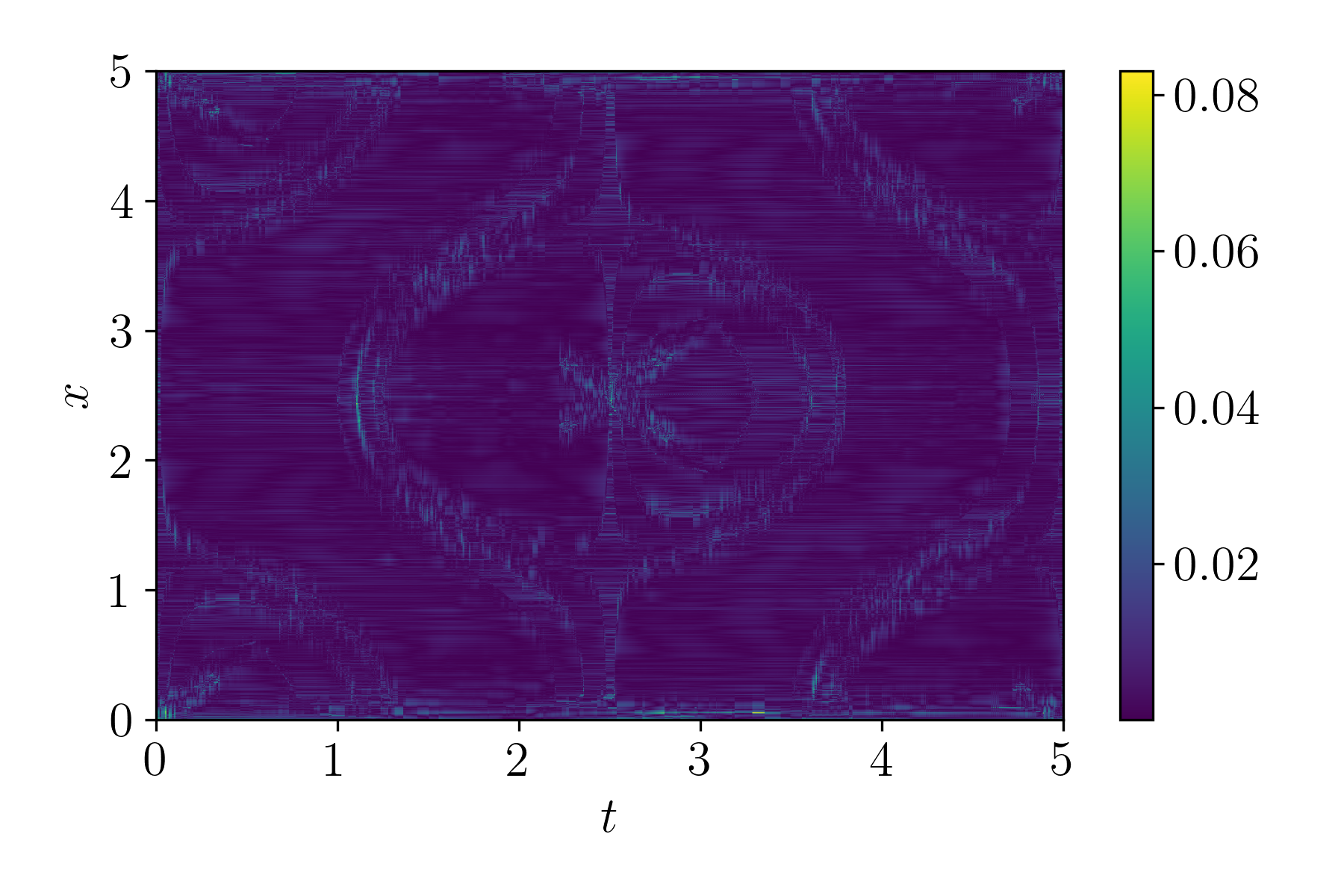}
		\caption{Absolute SympCAE error }
	\end{subfigure}
	\caption{Linear wave equation: Plot (a) shows the ground truth solution for state $q$. Plot (b) demonstrates the absolute pointwise error between the ground truth solution and the reconstructed solution obtained via PSD. Plot (c) shows the absolute pointwise error between the ground truth solution and the reconstructed solution obtained via the SympCAE.}
	\label{fig:wave-u}
\end{figure}

\begin{figure}[tb]
	\centering
	\begin{subfigure}{0.328\textwidth}
		\includegraphics[width=1\linewidth]{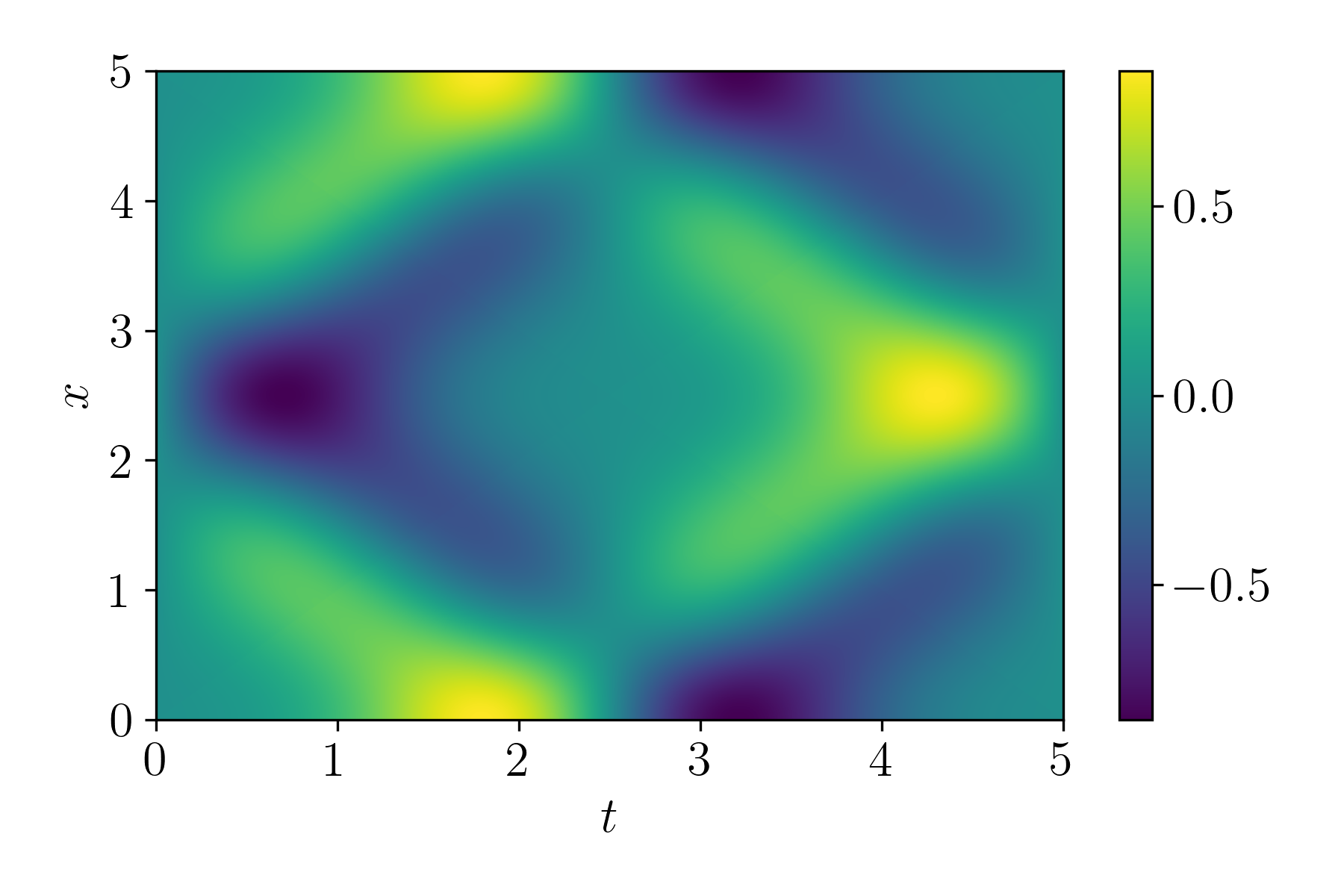}
		\caption{Ground truth}	
	\end{subfigure}
	\begin{subfigure}{0.328\textwidth}
		\includegraphics[width=1\linewidth]{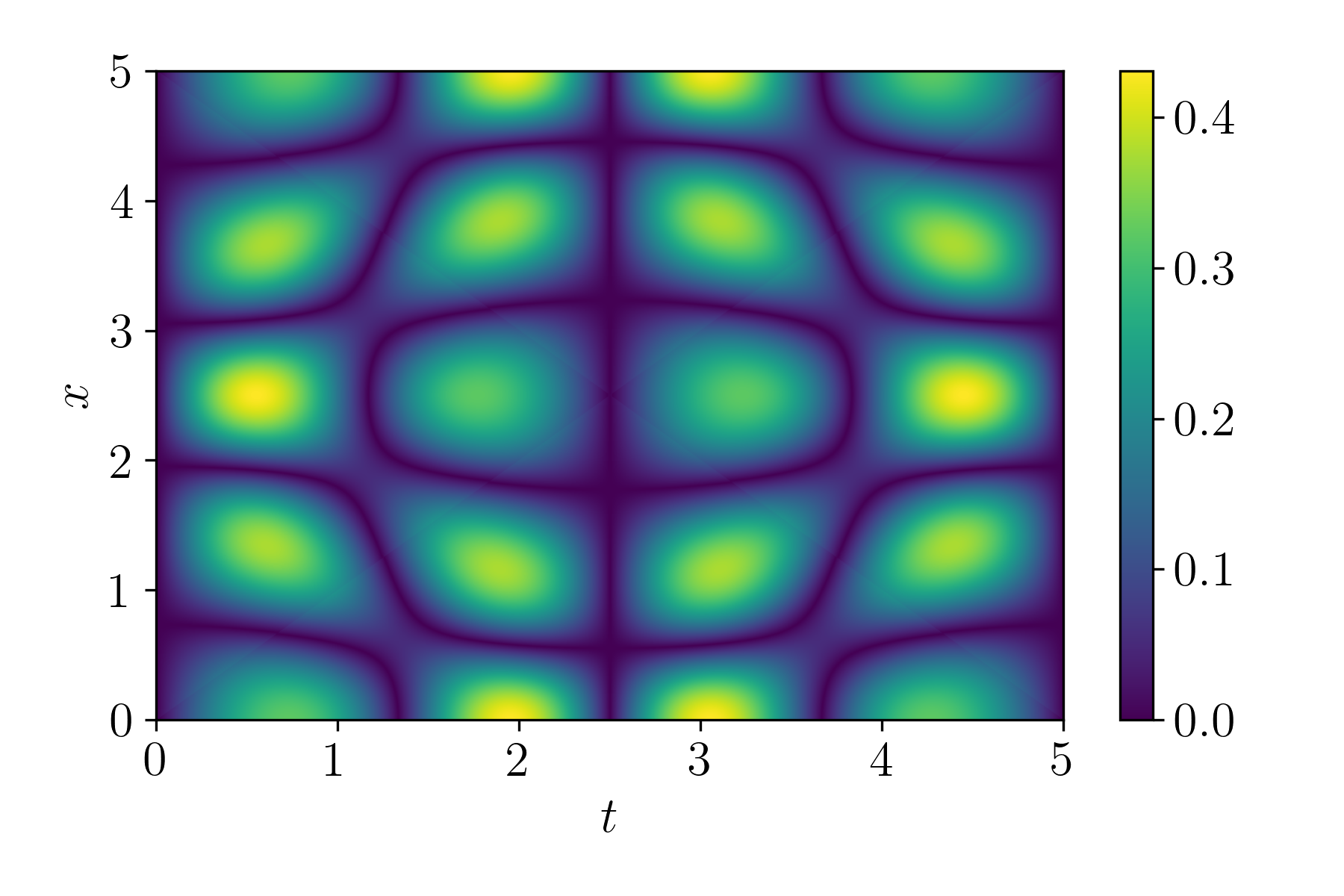} 
		\caption{Absolute PSD error}
	\end{subfigure}
	\begin{subfigure}{0.328\textwidth}
		\includegraphics[width=1\linewidth]{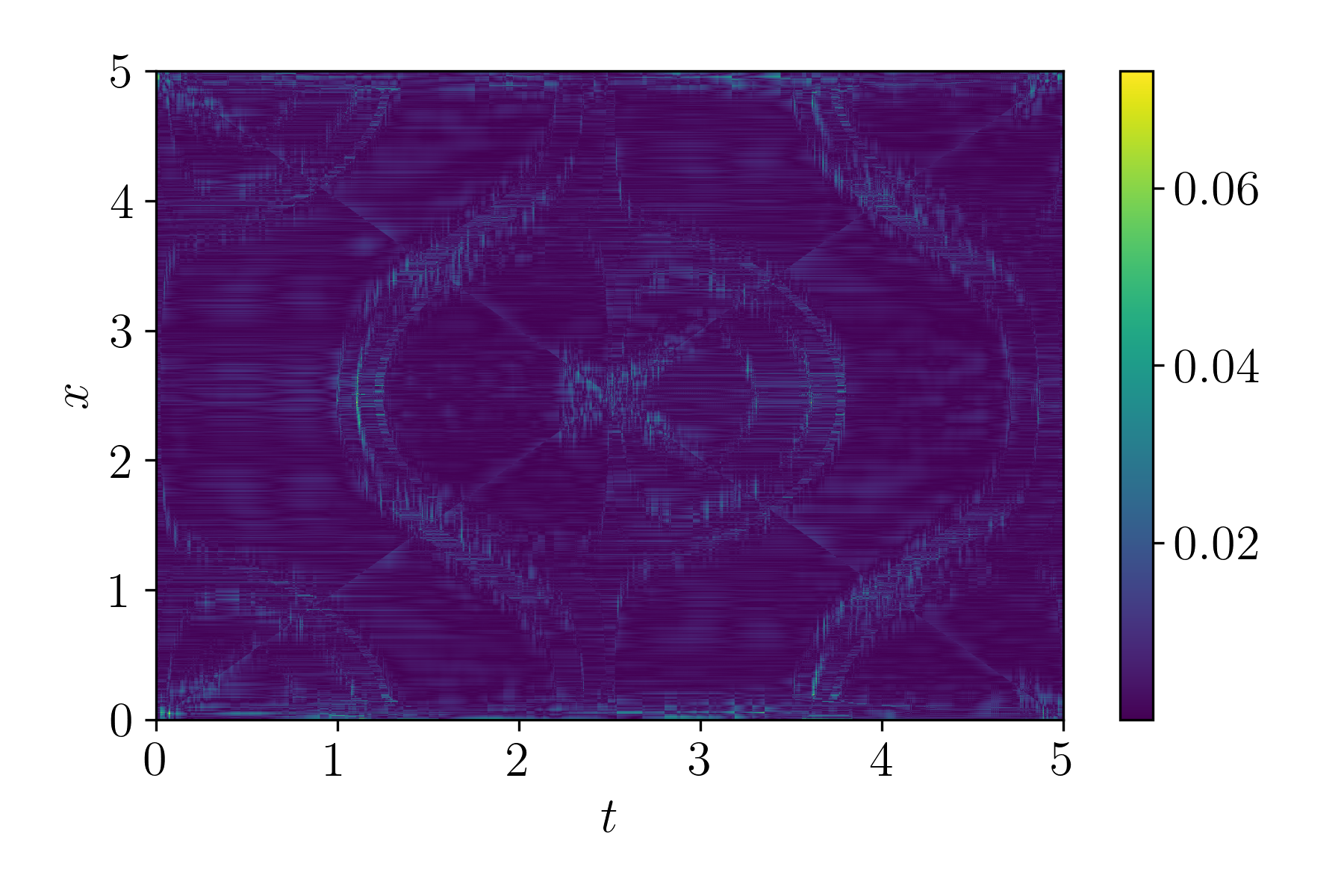}
		\caption{Absolute SympCAE error}
	\end{subfigure}
	\caption{Linear wave equation: Plot (a) shows the ground truth solution for the state $p$. Plot (b) demonstrates the absolute pointwise error between the ground truth solution and the reconstructed solution obtained via PSD. Plot (c) shows the absolute pointwise error between the ground truth solution and the reconstructed solution obtained via the SympCAE.}
	\label{fig:wave-v}
\end{figure}

Using the latent trajectories obtained via the encoder of the SympCAE, we train a SympNet to learn a low-dimensional symplectic model that can predict the dynamics of the wave equation beyond the training trajectories. To construct a training set for the SympNet, we used latent trajectories up to time $t = 5$. After learning a suitable model using the SympNet approach, we simulated the latent dynamics of the wave equation with the SympNet until $t = 10$ as the testing set, by applying the SympNet $t/\Delta t$ times to the latent initial condition. \Cref{fig:wave-SympNet} demonstrates the performance of the SympCAE combined with the SympNet. Specifically, \Cref{fig:wave-SympNet_latent} shows the latent trajectories for the training and testing sets, divided by a vertical line. We then reconstructed the testing trajectories via the decoder of the SympCAE and evaluated the performance of the combined approach using the reconstruction error \eqref{eqn:rel_time_err} in \Cref{fig:wave-SympNet_err}. This demonstrates the generalization of the combined method by showing that the SympNet is able to predict the latent testing trajectories for the SympCAE with good accuracy.

\begin{figure}[tb]
	\centering
	\begin{subfigure}{0.328\textwidth}
		\includegraphics[width=1\linewidth]{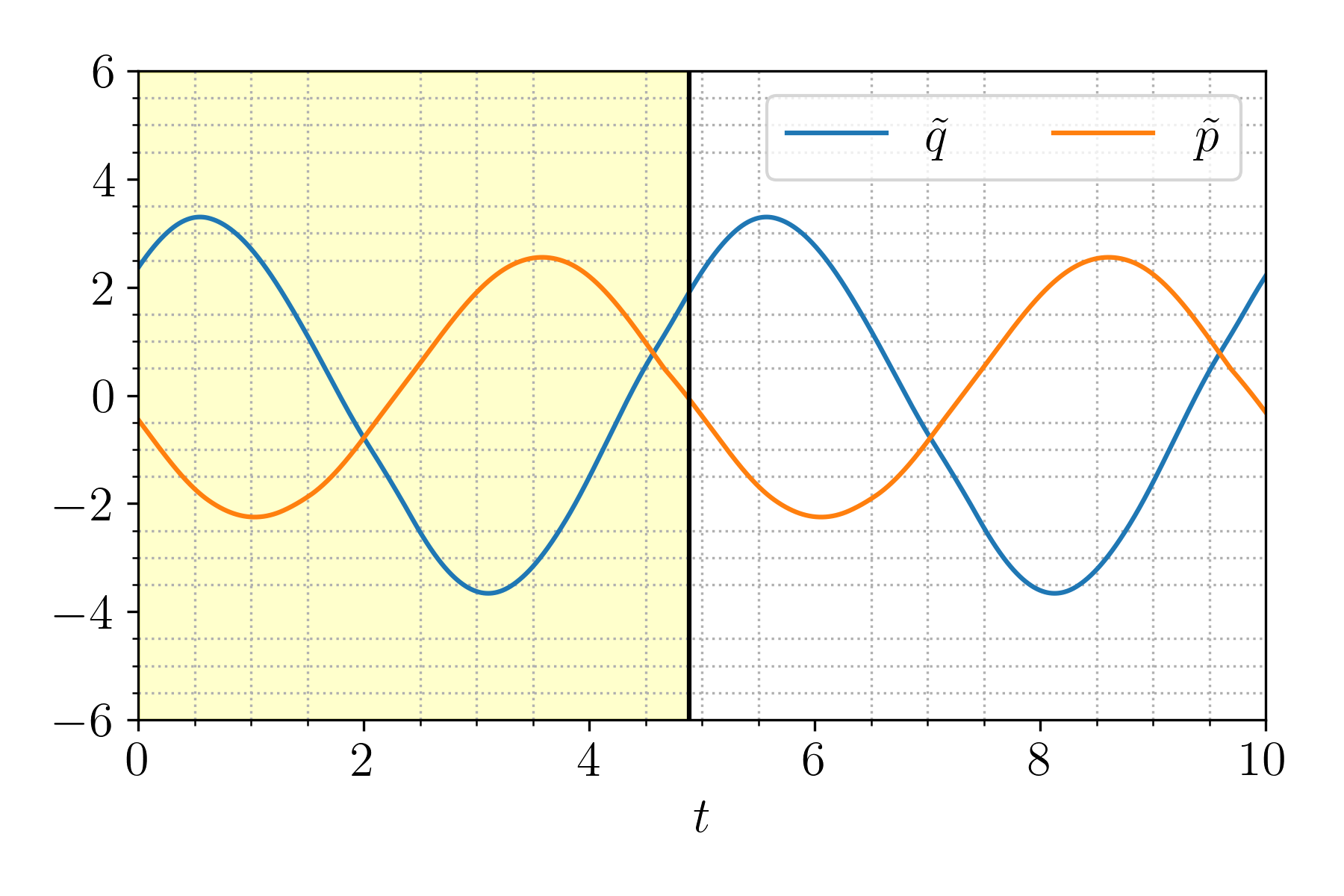}
		\caption{Latent trajectories}	
		\label{fig:wave-SympNet_latent}
	\end{subfigure}
	\begin{subfigure}{0.328\textwidth}
		\includegraphics[width=1\linewidth]{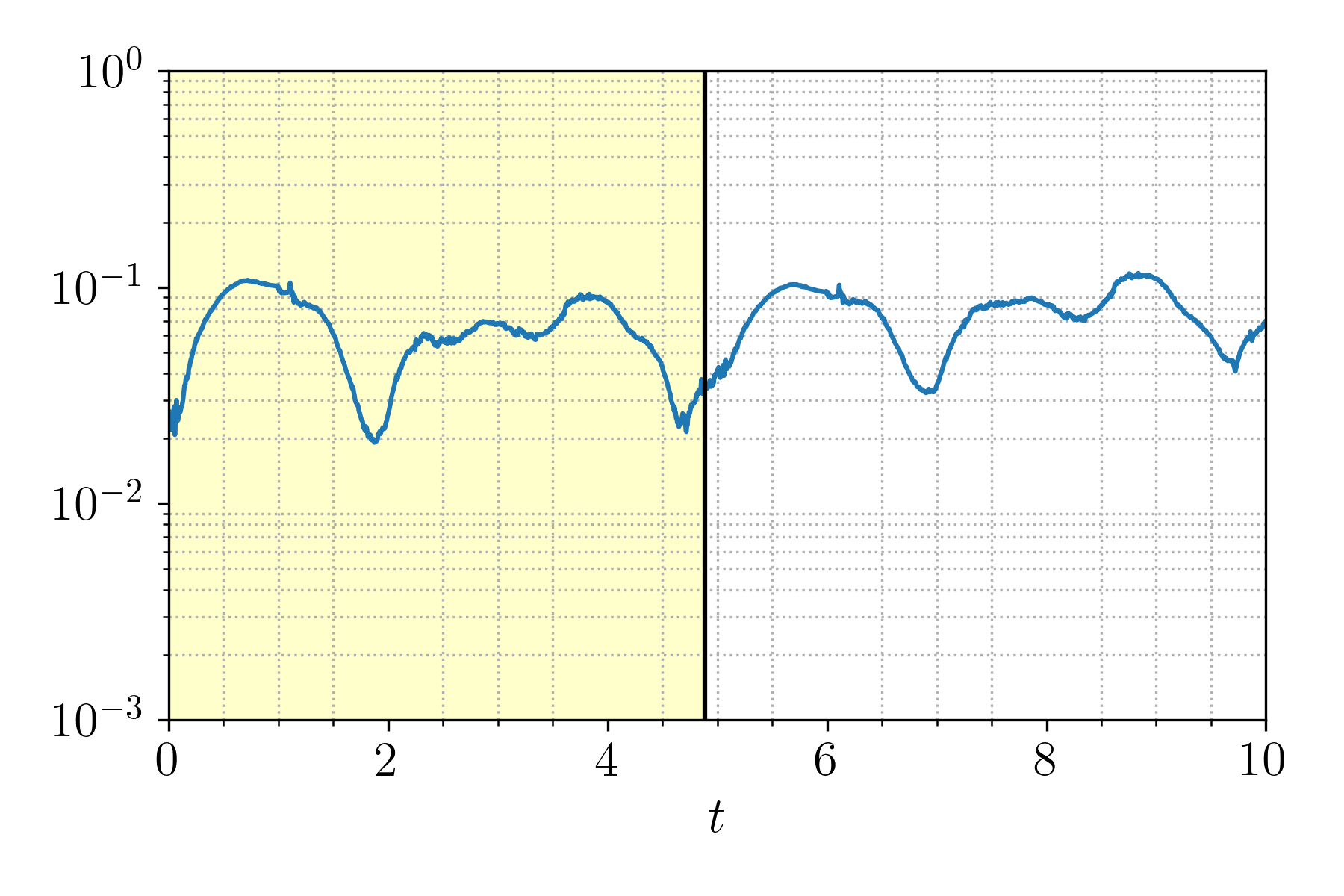}
		\caption{Relative reconstruction error}	
		\label{fig:wave-SympNet_err}
	\end{subfigure}
	\caption{A SympNet approach for learning the latent dynamics of the wave equation obtained by the SympCAE encoder. Plot (a) shows the latent trajectories obtained using the SympNet, and Plot (b) shows the relative reconstruction error over the time domain \eqref{eqn:rel_time_err} between the ground truth solution and the reconstructed solution obtained via decoder of the SympCAE. The vertical black line separate the training and testing intervals.}
	\label{fig:wave-SympNet}
\end{figure}
\subsection{NLS equation}
\label{subsec:NLS}
In our second example, following~\cite{yildiz2024data}, we test the proposed symplectic autoencoder on learning the nonlinear Schrödinger (NLS) equation. This equation is fundamental in studying various phenomena, such as Bose-Einstein condensation, small-amplitude deep-water gravity waves with zero viscosity, and the propagation of light in nonlinear optical fibers. The one-dimensional cubic NLS equation is given by:
\begin{equation}\label{eqn:schrod}
	\begin{aligned}
		& \imath  u_t(x,t) + \alpha u_{xx}(x,t) + \beta |u(x,t)|^2 u(x,t) = 0, \\
		& u(x,0) = u^0(x), & x \in \Omega.
	\end{aligned}
\end{equation}
Here, the constant parameter $\beta$ describes if the nonlinearity of the NLS equation is focusing for negative values and defocusing for positive values, and the parameter $\alpha$ is a non-negative constant. In this example, we set the parameters to $\alpha=1$, $\beta=1.5$, the initial condition to $u^0(x)=\sqrt{2}\sech(x)$, and consider the domain $\Omega=[-2\pi,2\pi]$ with periodic boundary conditions.

The canonical Hamiltonian form of the NLS equation \eqref{eqn:schrod} appears after expressing the complex-valued solution $u$ in terms of its imaginary and real parts as $u=p+ \imath q$, which yields:
\begin{equation}
	\begin{aligned}
		&q_t = p_{xx}  + \beta(p^2+q^2)p,\\
		&p_t = -q_{xx} - \beta(p^2+q^2)q,
	\end{aligned}
\end{equation}
with the Hamiltonian
\begin{align*}
	\mathcal{H}(u) = \dfrac{1}{2}\int_{\Omega}
	(q_x^2 + p_x^2) + \dfrac{\beta}{2} (q^2 + p^2)^2 \, dx.
\end{align*}
To obtain a structure-preserving discretization in space, we discretize the space using the central difference approach as in the previous example, which yields the following system of ODEs:
\begin{equation}\label{eqn:nls-semi-discrete}
	\begin{aligned}
		&\ve{q}_t = \ma{D}_{xx}\ve{p} + \beta\left(\ve{p}\odot\ve{p} + \ve{q}\odot\ve{q}\right)\odot\ve{p}, \\
		&\ve{p}_t = -\ma{D}_{xx}\ve{q} + \beta\left(\ve{p}\odot\ve{p} + \ve{q}\odot\ve{q}\right)\odot\ve{q},
	\end{aligned}
\end{equation}  
where $\odot$ denotes the element-wise (Hadamard) product and $\ve{p}, \ve{q} \in \mathbb{R}^{N}$ are discretized uniformly as in the previous example. We obtain the training data by integrating the semi-discrete NLS equation~\eqref{eqn:nls-semi-discrete} using the implicit midpoint rule, which is an implicit symplectic integrator~\cite{hairer2006structure}. As in the previous example, we use $N=1024$ grid points in the spatial domain and simulate the trajectories of the NLS equation~\eqref{eqn:schrod} until time $t=5$ with $N_t=200$ time points, which yields the training set $\ma{X}\in \mathbb{R}^{2048\times 200}$. 

As in the previous example, we first train our symplectic autoencoder using the reshaped input data  $\te{X}\in \mathbb{R}^{N_t\times2\times N}$. We demonstrate the relative reconstruction error~\eqref{eqn:rel_err} in \Cref{tab:NLS} for the first three latent dimensions and compare it with the PSD low-rank approximation. Similar to the previous example, this shows that the SympCAE is outperforming the PSD autoencoder. Nevertheless, the SympCAE model for latent dimension $r=3$ performs slightly worse compared to latent dimension $r=2$ due to the fixed selected hyperparameters for the three different dimensions. Obviously, a variation of hyperparameters tailored towards $r=3$ would lead to better performance in that case. In~\Cref{fig:NLS-u,fig:NLS-v}, we show the ground truth trajectory of the NLS equation \eqref{eqn:schrod} and compare the SympCAE with the low-rank solution obtained via PSD in terms of pointwise absolute error for latent dimension $r=1$, which again shows that the proposed autoencoder can capture the dynamics of the NLS equation more accurately.

Lastly, we present the encoded latent trajectories obtained through the encoder of the SympCAE, along with the reconstructed solution achieved by combining a SympNet with the decoder of the SympCAE, in \Cref{fig:NLS-SympNet}. To train the SympNet model, we first encoded the input data $\te{X}\in \mathbb{R}^{N_t\times2\times N}$ using the SympCAE. We then simulated the SympNet model until time $t=10$, as shown in \Cref{fig:NLS-SympNet_latent}. Finally, we reconstructed the latent testing trajectories using the decoder of SympCAE. \Cref{fig:NLS-SympNet_err} presents the relative reconstruction \eqref{eqn:rel_time_err}, which indicates that the testing trajectories perform slightly worse than those in the training set. Note that we have fixed the indices of the pooling layers while training the autoencoder to preserve the symplectic structure and used the same indices in the latent testing trajectories, which results in relatively lower accuracy. One potential solution to avoid this issue is to remove the symplectic pooling layer from the autoencoder. Nevertheless, since the error does not deviate significantly, the proposed architecture remains with an acceptable error in this context.

\begin{table}[tb]
	\caption{NLS equation: The table shows the performance of the autoencoder obtained using PSD and the SympCAE in capturing the dynamics of the ground truth model in terms of relative reconstruction error \eqref{eqn:rel_err} at latent dimensions $r = 1, 2, 3$. The best result for each latent dimension is highlighted in bold.}
	\label{tab:NLS}
	\centering
	\renewcommand{\arraystretch}{1.2}
	\begin{tabular}{lll}
		$r$  &  $\bm{\varepsilon}_{\text{PSD}}$ & $\bm{\varepsilon}_{\text{SympCAE}}$ \\ \hline
		$1$ & $1.85\cdot 10^{-1}$  & $\mathbf{4.44\cdot 10^{-2}}$   \\ 
		$2$ & $1.04\cdot 10^{-1}$  & $\mathbf{1.35\cdot 10^{-2}}$    \\ 
		$3$ & $5.21\cdot 10^{-2}$  & $\mathbf{1.82\cdot 10^{-2}}$   \\ \hline
	\end{tabular}
\end{table}

\begin{figure}[tb]
	\centering
	\begin{subfigure}{0.328\textwidth}
		\includegraphics[width=1\linewidth]{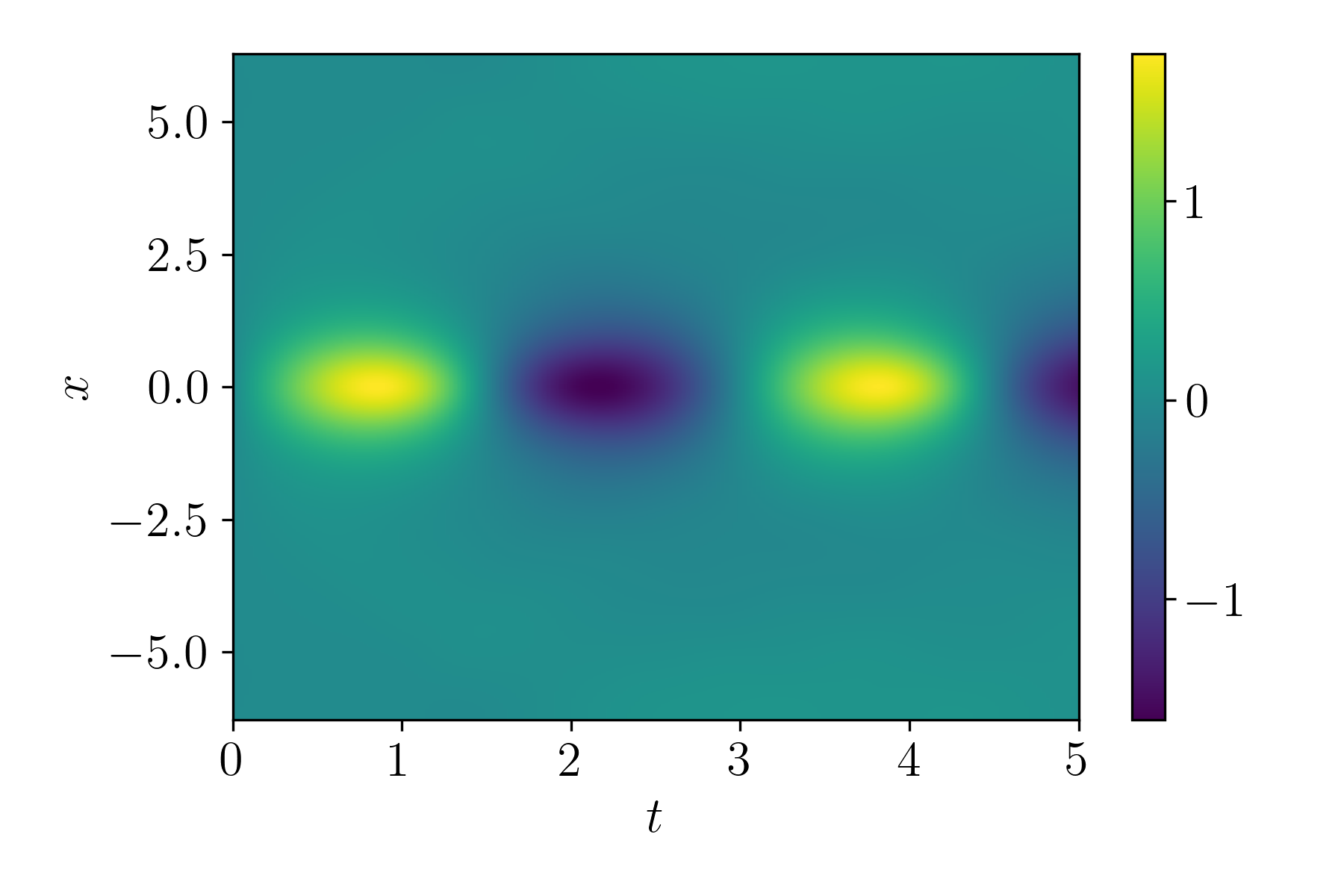}
		\caption{Ground truth}	
	\end{subfigure}
	\begin{subfigure}{0.328\textwidth}
		\includegraphics[width=1\linewidth]{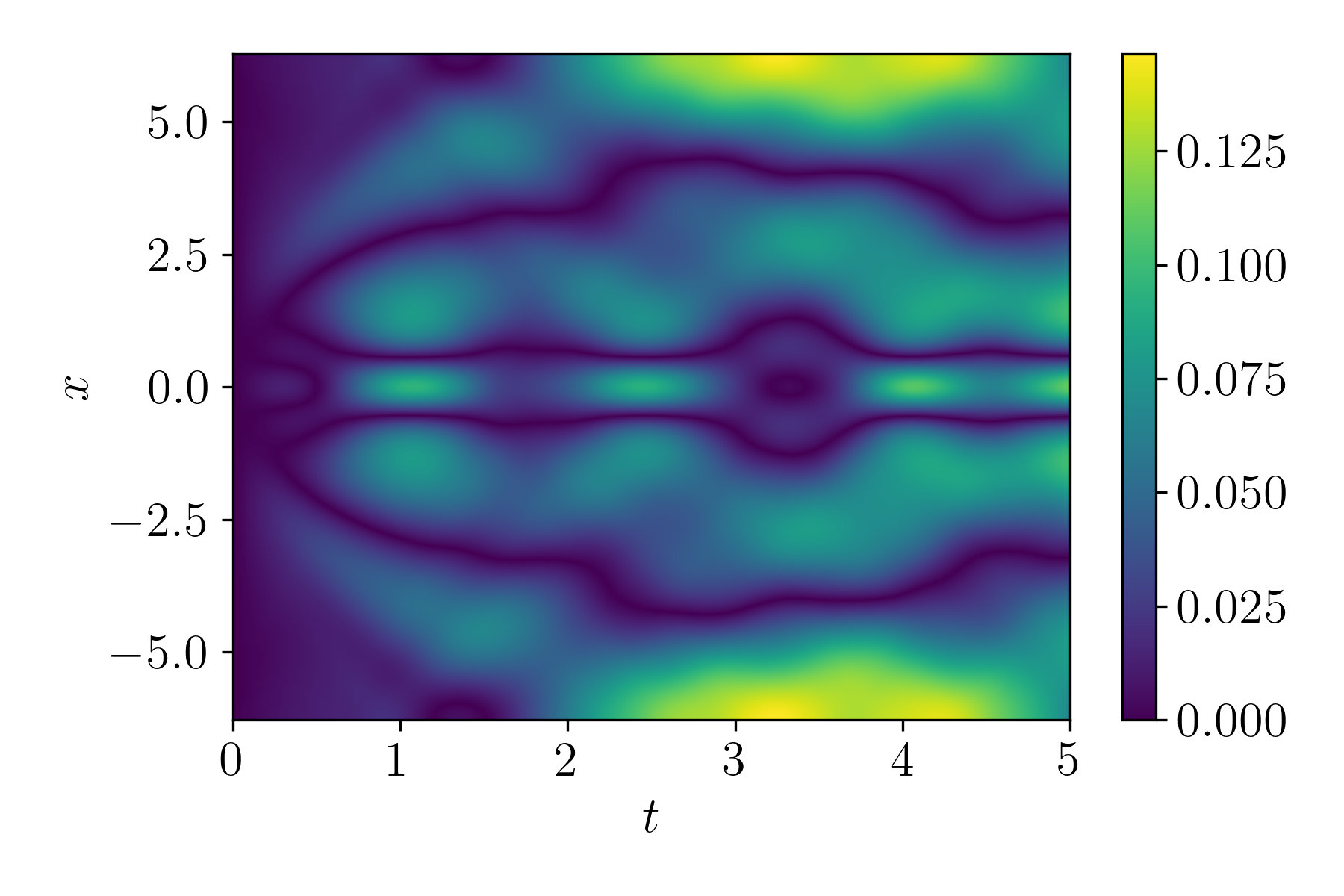} 
		\caption{Absolute error PSD}
	\end{subfigure}
	\begin{subfigure}{0.328\textwidth}
		\includegraphics[width=1\linewidth]{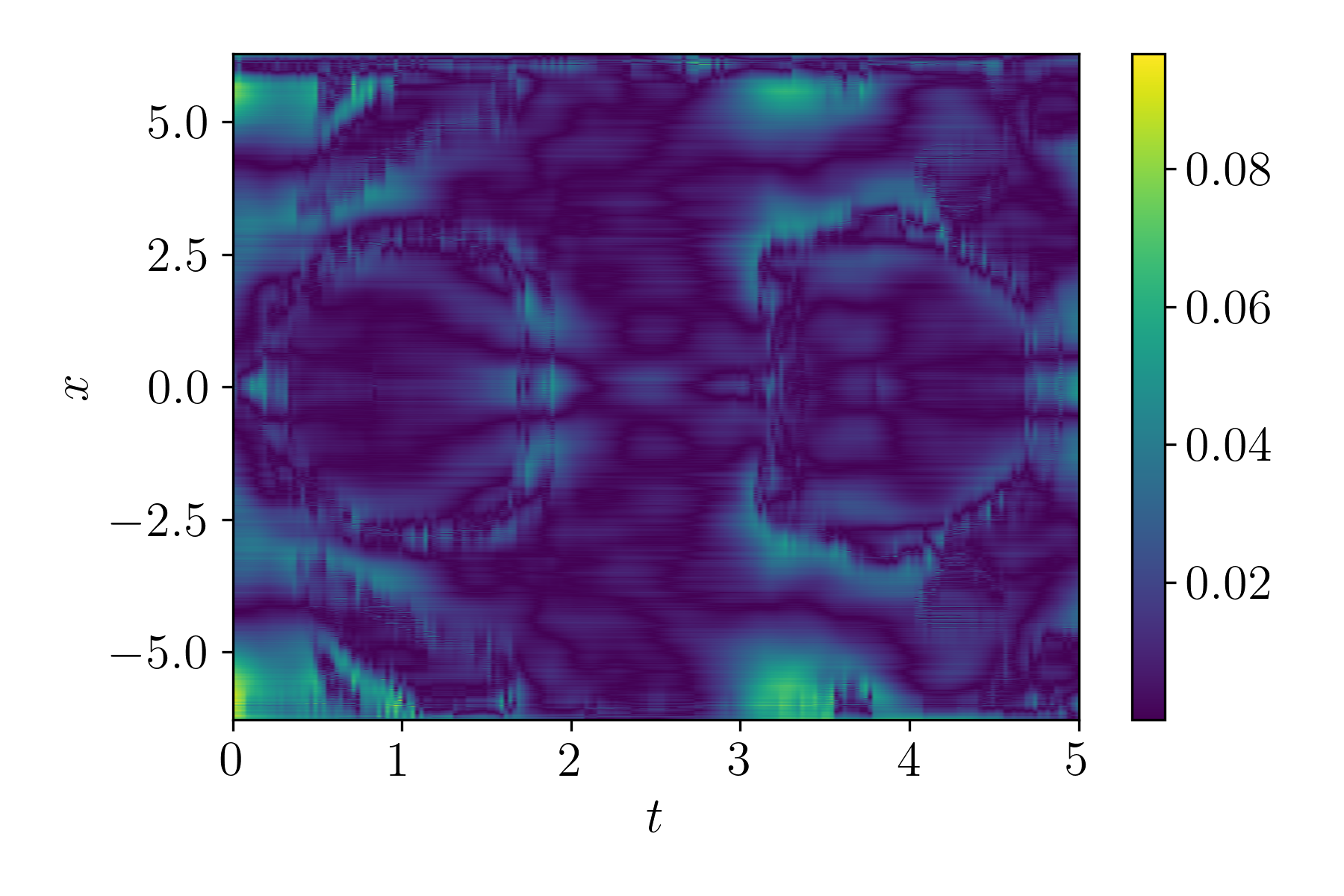}
		\caption{Absolute error CNN}
	\end{subfigure}
	\caption{NLS equation: Plot (a) shows the ground truth solution  for state $q$. Plot (b) demonstrates the absolute pointwise error between the ground truth solution and the reconstructed solution obtained via PSD. Plot (c) shows the absolute pointwise error between the ground truth solution and the reconstructed solution obtained via the SympCAE.}
	\label{fig:NLS-u}
\end{figure}

\begin{figure}[tb]
	\centering
	\begin{subfigure}{0.328\textwidth}
		\includegraphics[width=1\linewidth]{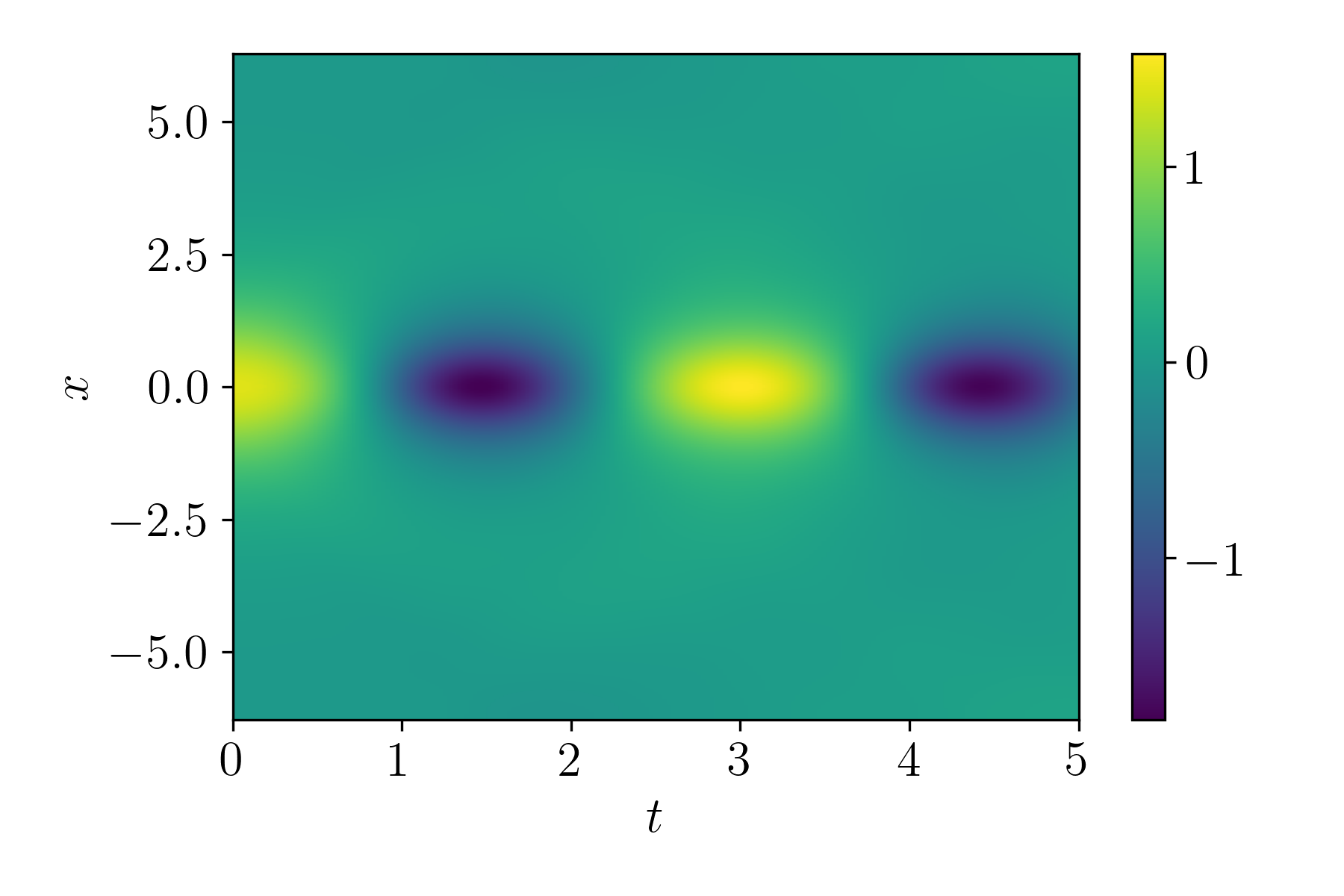}
		\caption{Ground truth}	
	\end{subfigure}
	\begin{subfigure}{0.328\textwidth}
		\includegraphics[width=1\linewidth]{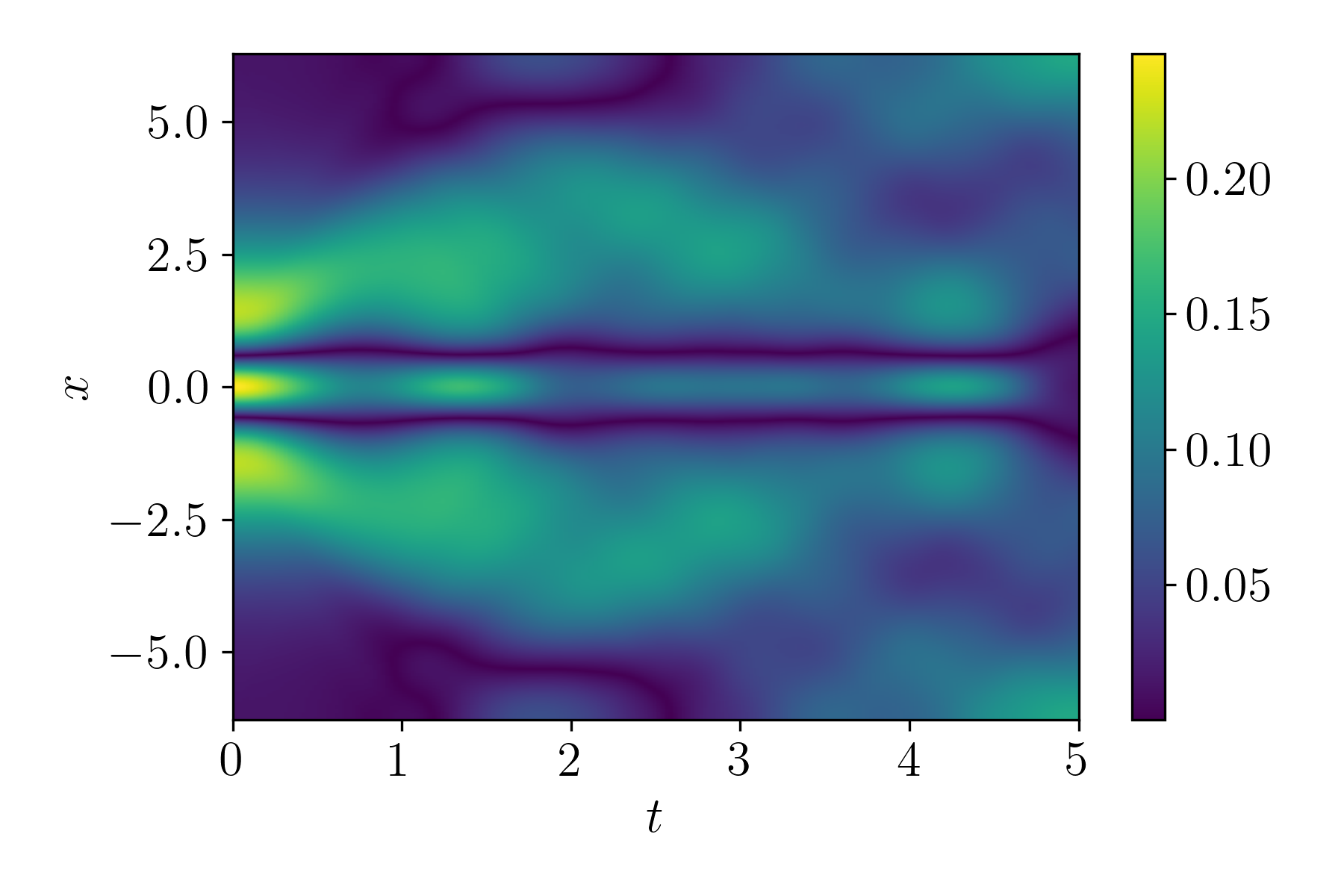} 
		\caption{Absolute error PSD}
	\end{subfigure}
	\begin{subfigure}{0.328\textwidth}
		\includegraphics[width=1\linewidth]{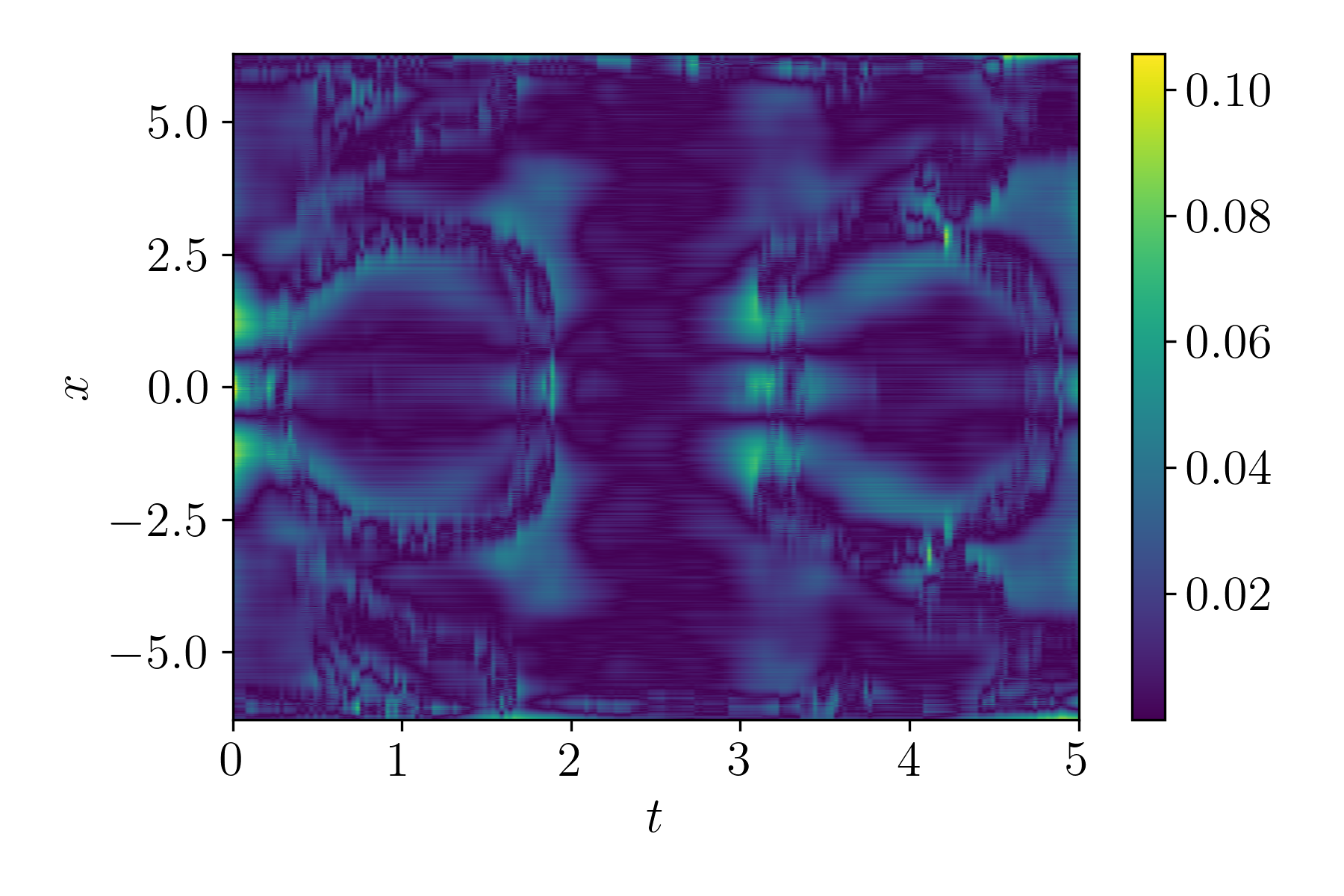}
		\caption{Absolute error CNN}
	\end{subfigure}
	\caption{NLS equation: Plot (a) shows the ground truth solution for state $p$. Plot (b) demonstrates the absolute pointwise error between the ground truth solution and the reconstructed solution obtained via PSD. Plot (c) shows the absolute pointwise error between the ground truth solution and the reconstructed solution obtained via the SympCAE.}
	\label{fig:NLS-v}
\end{figure}

\begin{figure}[tb]
	\centering
	\begin{subfigure}{0.328\textwidth}
		\includegraphics[width=1\linewidth]{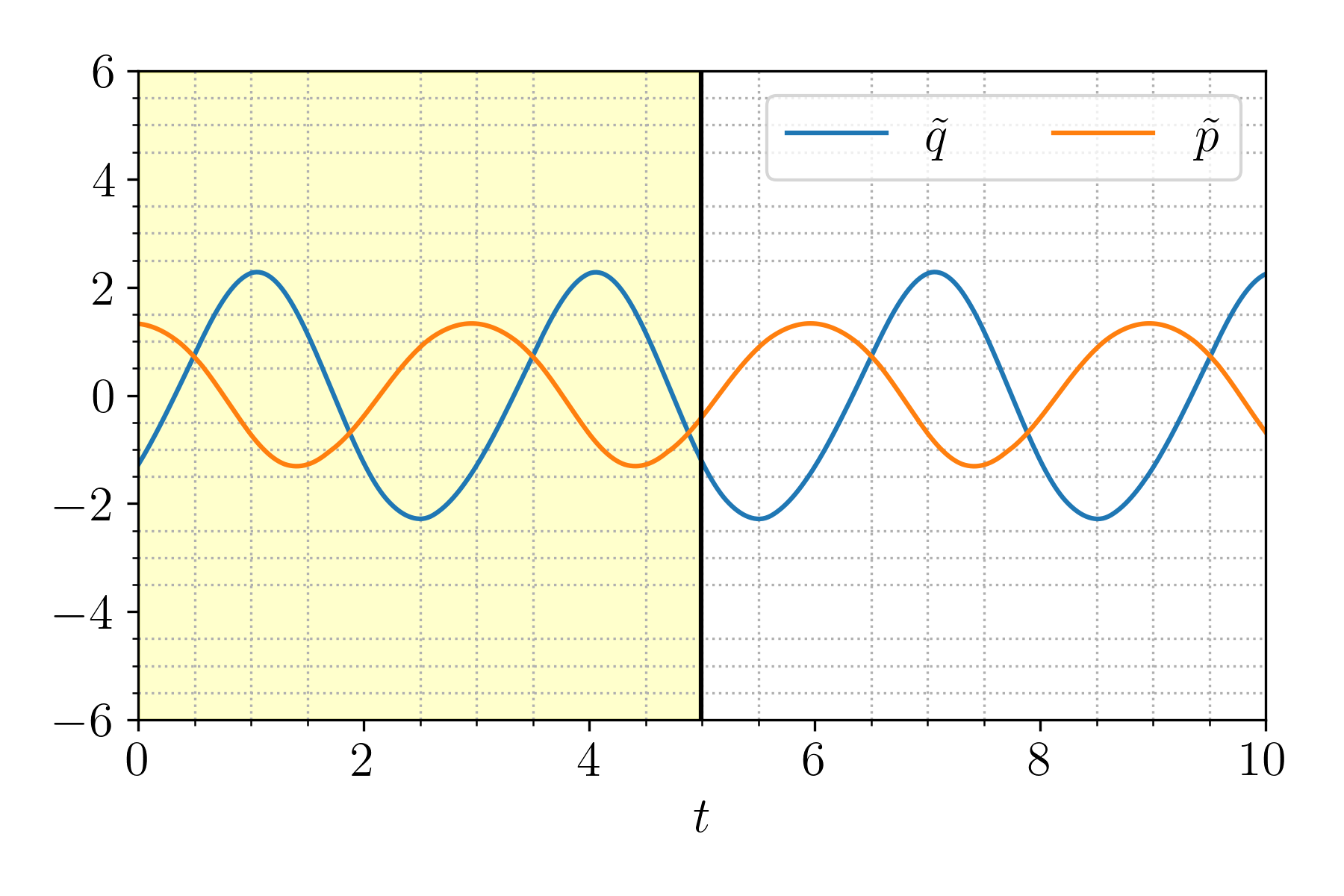}
		\caption{Latent trajectories}
		\label{fig:NLS-SympNet_latent}	
	\end{subfigure}
	\begin{subfigure}{0.328\textwidth}
		\includegraphics[width=1\linewidth]{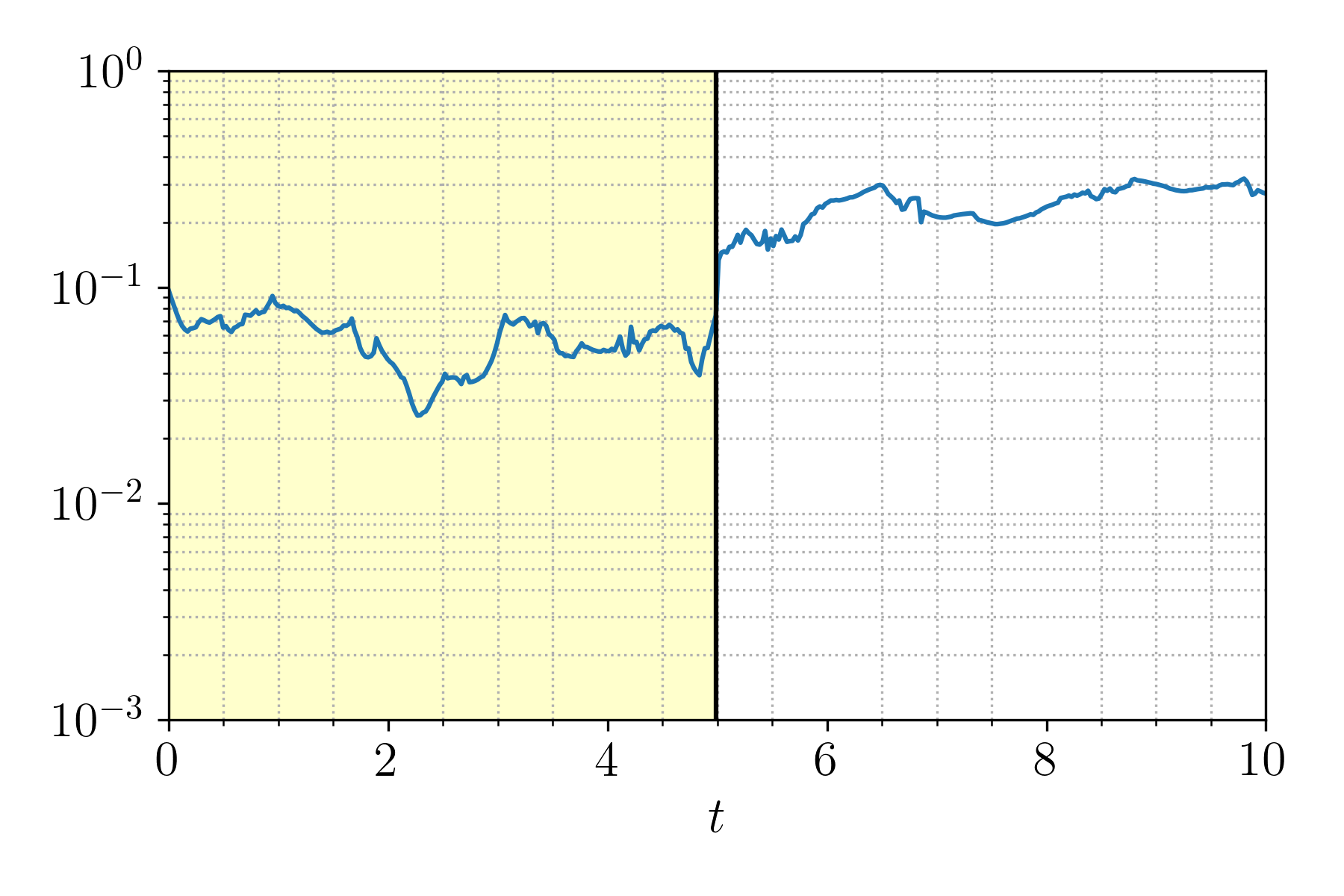}
		\caption{Relative reconstruction error}
		\label{fig:NLS-SympNet_err}	
	\end{subfigure}
	\caption{A SympNet approach for learning the latent dynamics of the NLS equation obtained by the SympCAE encoder. Plot (a) shows the latent trajectories obtained using the SympNet, and Plot (b) shows the relative reconstruction error over the time domain \eqref{eqn:rel_time_err} between the ground truth solution and the reconstructed solution obtained via decoder of the SympCAE. The vertical black line separates the training and testing intervals.}
	\label{fig:NLS-SympNet}
\end{figure}

\subsection{Sine-Gordon equation}
\label{subsec:SG}
In our last example, we consider the two-dimensional sine-Gordon equation given as follows:
\begin{equation}\label{eqn:sine-gordon}
	\begin{aligned}
		& u_{tt}(x,y,t) =  u_{xx}(x,y,t) + u_{yy}(x,y,t) - \sin(u(x,y,t)), \\
		& u(x,y,0) = u^0(x,y), \\
		& u_t(x,y,0) = u_t^0(x,y), \quad x,y \in \Omega.
	\end{aligned}
\end{equation}
Following \cite{sharma2025nonlinear}, we set the boundary condition to be periodic as in previous examples, with the spatial domain $\Omega=(-7,7)\times(-7,7)$ and the initial conditions
\begin{equation*}
	u^0(x,y)=4\tan^{-1}\left(\exp\left(3-\sqrt{x^2+y^2}\right)\right), \quad u_t^0(x,y)=0 .
\end{equation*}

We discretize the spatial domain using a structure-preserving central difference approach with $N_x=N_y=100$ equally spaced grid points in both the $x$ and $y$ directions, which yields a spatially discrete state $\ve{u} \in \mathbb{R}^{10000}$. We obtain the Hamiltonian form of the SG equation by introducing $q(x,y,t)=u(x,y,t)$ and $p(x,y,t)=u_t(x,y,t)$, which leads to the following conservative form:
\begin{equation*}
	\begin{aligned}
		&q_t(x,y,t) = p(x,y,t),\\
		&p_t(x,y,t) = q_{xx}(x,y,t) + q_{yy}(x,y,t) - \sin(q(x,y,t)).
	\end{aligned}
\end{equation*}
Moreover, after spatial discretization, the spatially discrete SG equation is given as follows:
\begin{equation}\label{eqn:SG-semi-discrete}
	\begin{aligned}
		&\ve{q}_t = \ve{p},\\
		&\ve{p}_t = \ma{D}_{xx}\ve{q} + \ma{D}_{yy}\ve{q} - \sin(\ve{q}),
	\end{aligned}
\end{equation} 
where $\ma{D}_{xx}, \ma{D}_{yy} \in \mathbb{R}^{10000\times 10000}$ are the central difference approximations of the partial derivatives $\partial_{xx}$ and $\partial_{yy}$, respectively. The semi-discrete equations \eqref{eqn:SG-semi-discrete} conserve the following spatially discrete Hamiltonian:
\begin{equation*}
H=\frac{1}{2}\left(\ve{p}^T\ve{p}-\ve{q}^T\ma{D}\ve{q}\right)+\sum_{i=1}^{N}(1-cos(\ve{q}_i)),
\end{equation*}
where $\ma{D}=\ma{D}_{xx}+\ma{D}_{yy}$.
To preserve the symplectic structure in time, we discretize the semi-discrete SG equation \eqref{eqn:SG-semi-discrete} using the implicit midpoint rule, as in the previous example. Moreover, we constructed the training set by sampling $N_t=100$ time points over the time domain $[0,20]$. To train the 2D SympCAE, we reshaped the training set as $\te{X}\in \mathbb{R}^{N_t\times2\times N_x\times N_y}$. 

In~\Cref{tab:SG}, we compare the low-rank solutions obtained via the 2D SympCAE and the PSD autoencoder for the first three latent dimensions. The table indicates that the proposed approach yields better accuracy in terms of reconstruction error \eqref{eqn:rel_err}. Furthermore, the table shows a very slow decay in relative error for the low-rank PSD approximation, which is a common characteristic problem in Hamiltonian dynamics. Lastly, we show the ground truth and approximate low-rank solutions for latent dimension $r=3$ at time $t=20$ obtained via PSD and symplectic CNN in \Cref{fig:SG-u,fig:SG-v}, demonstrating that the proposed autoencoder is capable of capturing the dynamics of the SG equation with a very small latent dimension.

\begin{table}[tb]
	\caption{SG equation: The table shows the performance of the autoencoder obtained by using PSD and the SympCAE in capturing the dynamics of the ground truth model in terms of relative reconstruction error \eqref{eqn:rel_err} at latent dimensions $r = 1, 2, 3$. The best result for each latent dimension is highlighted in bold.}
	\label{tab:SG}
	\centering
	\renewcommand{\arraystretch}{1.2}
	\begin{tabular}{lll}
		$r$  &  $\bm{\varepsilon}_{\text{PSD}}$ & $\bm{\varepsilon}_{\text{SympCAE}}$ \\ \hline
		$1$ & $3.74\cdot 10^{-1}$  & $\mathbf{1.35\cdot 10^{-1}}$   \\ 
		$2$ & $3.07\cdot 10^{-1}$  & $\mathbf{5.15\cdot 10^{-2}}$    \\ 
		$3$ & $2.55\cdot 10^{-1}$  & $\mathbf{7.86\cdot 10^{-2}}$ \\ \hline  
	\end{tabular}
\end{table}

\begin{figure}[tb]
	\centering
	\begin{subfigure}{0.328\textwidth}
		\includegraphics[width=1\linewidth]{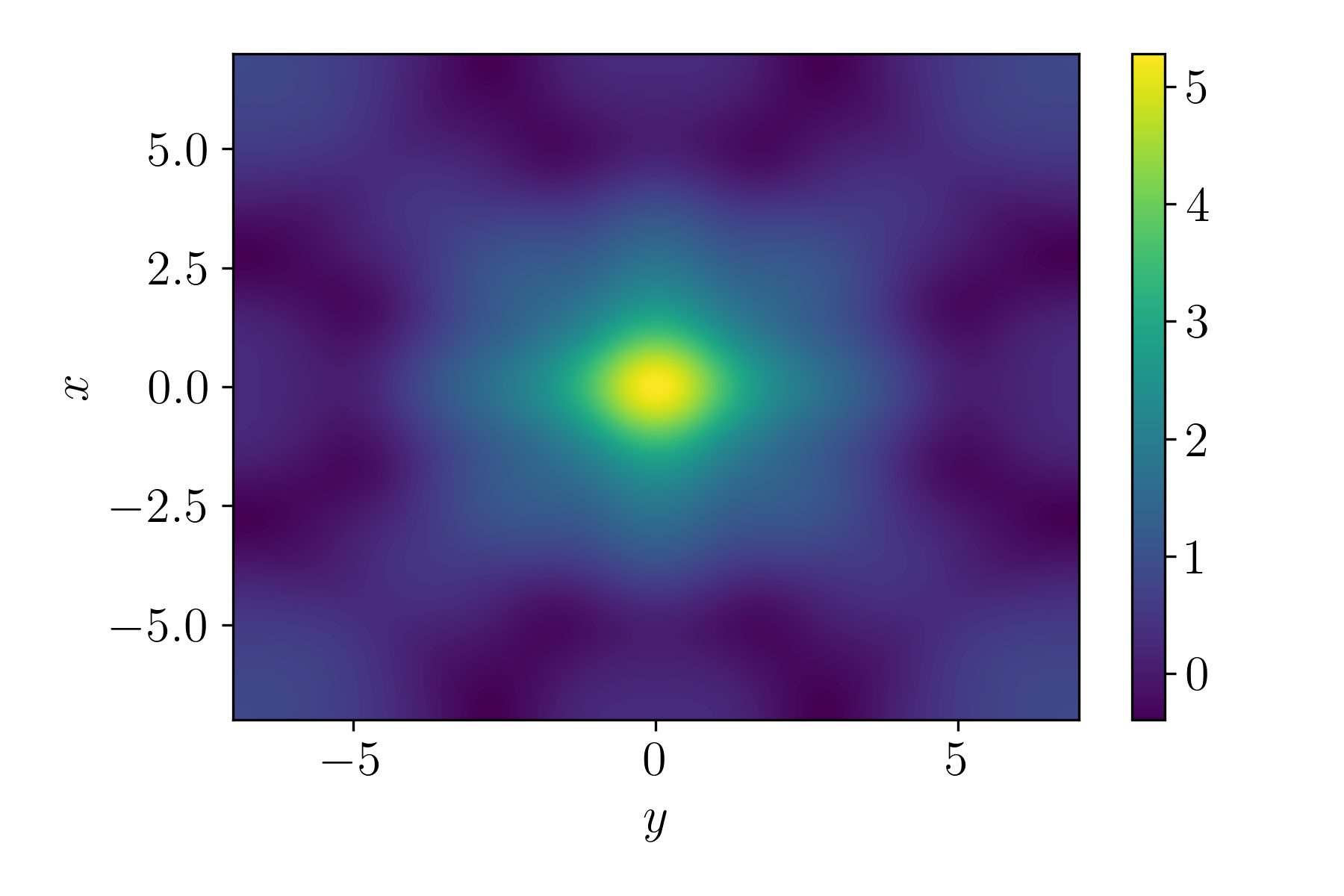}
		\caption{Ground truth}	
	\end{subfigure}
	\begin{subfigure}{0.328\textwidth}
		\includegraphics[width=1\linewidth]{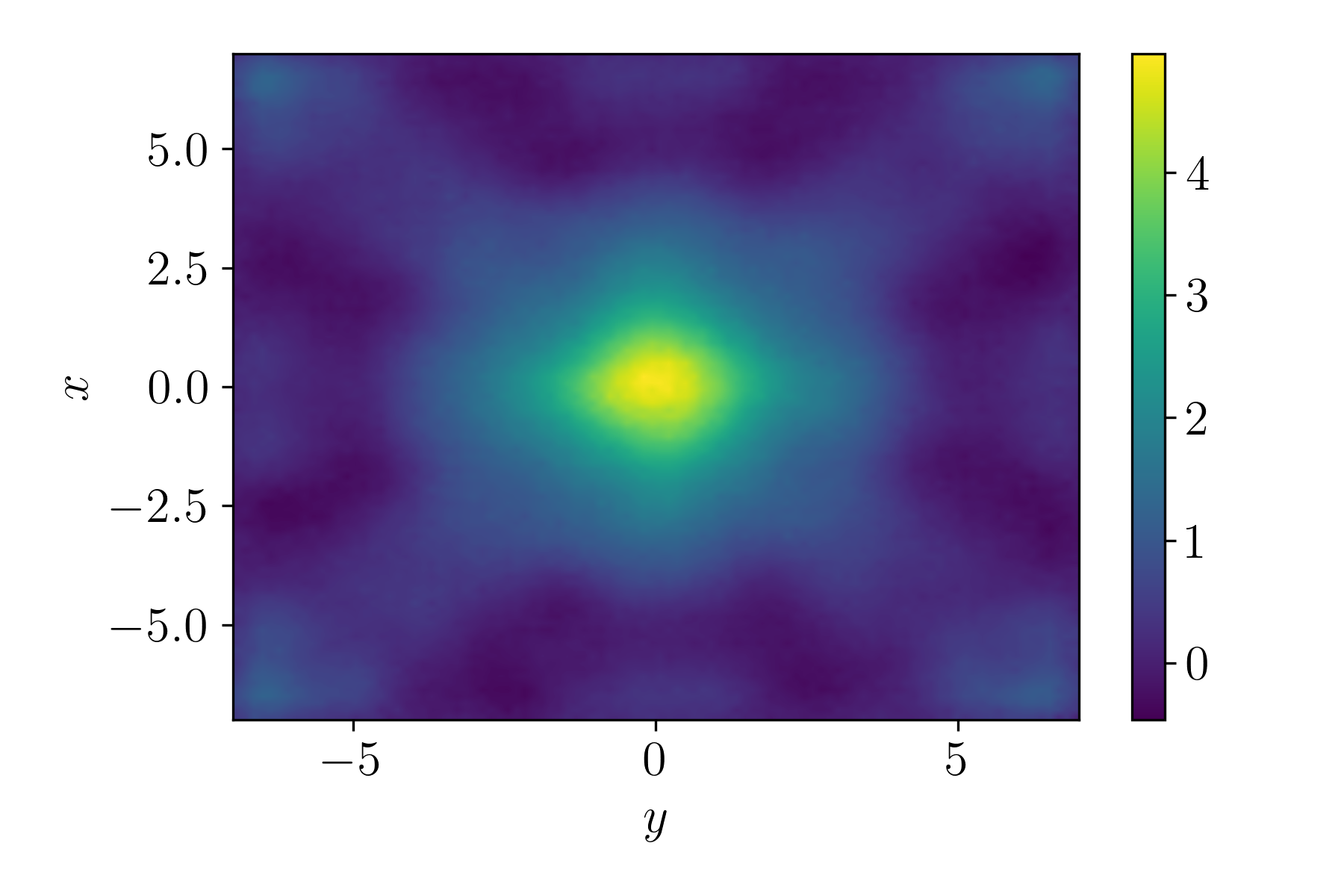} 
		\caption{SympCAE}
	\end{subfigure}
	\begin{subfigure}{0.328\textwidth}
		\includegraphics[width=1\linewidth]{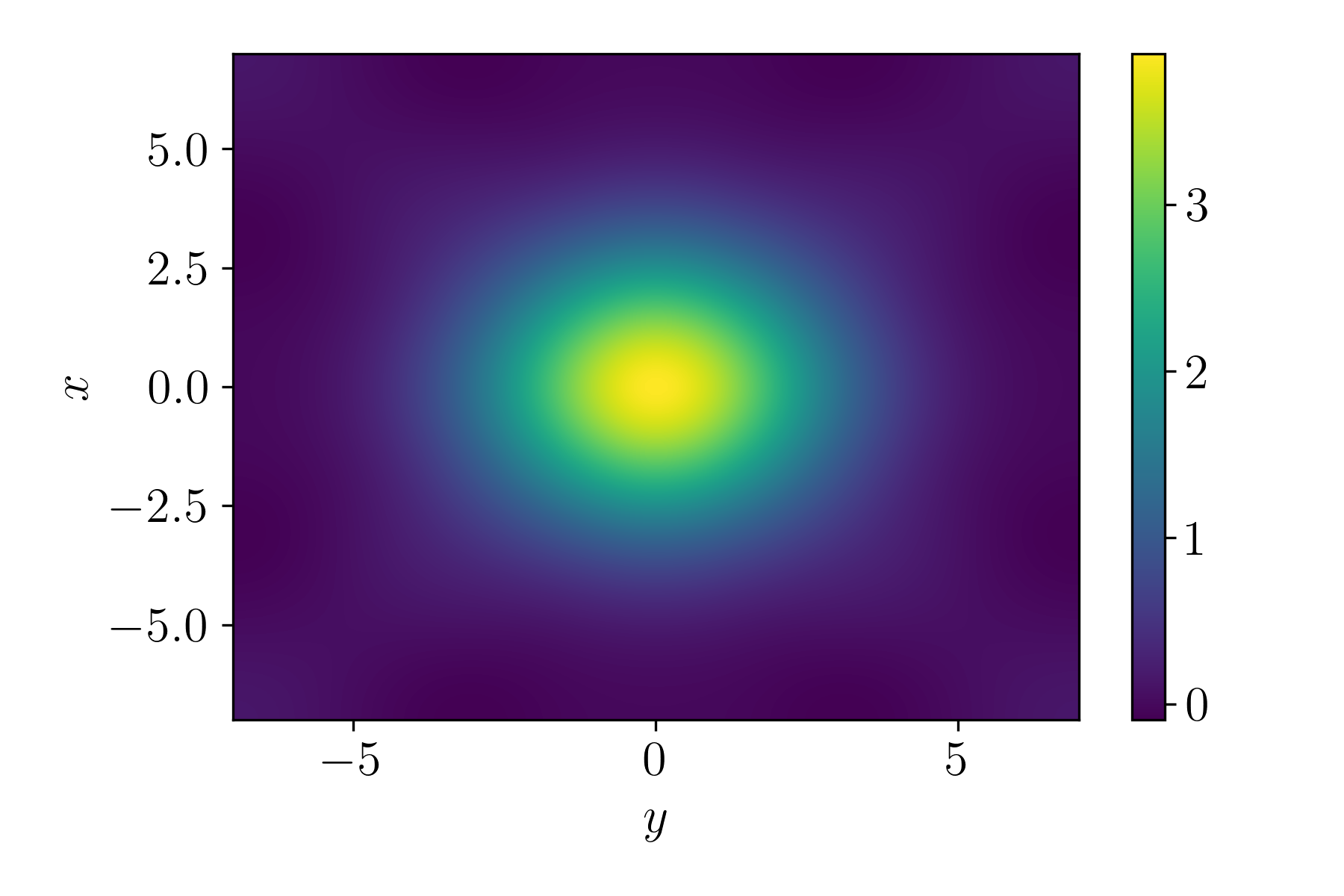}
		\caption{PSD}
	\end{subfigure}
	\caption{SG equation: Comparison of the solutions of $ q $ at final time $ t = 20 $ for latent time dimension $ r = 3 $. (a) Ground truth, (b) SympCAE, (c) PSD. }
	\label{fig:SG-u}
\end{figure}

\begin{figure}[tb]
	\centering
	\begin{subfigure}{0.328\textwidth}
		\includegraphics[width=1\linewidth]{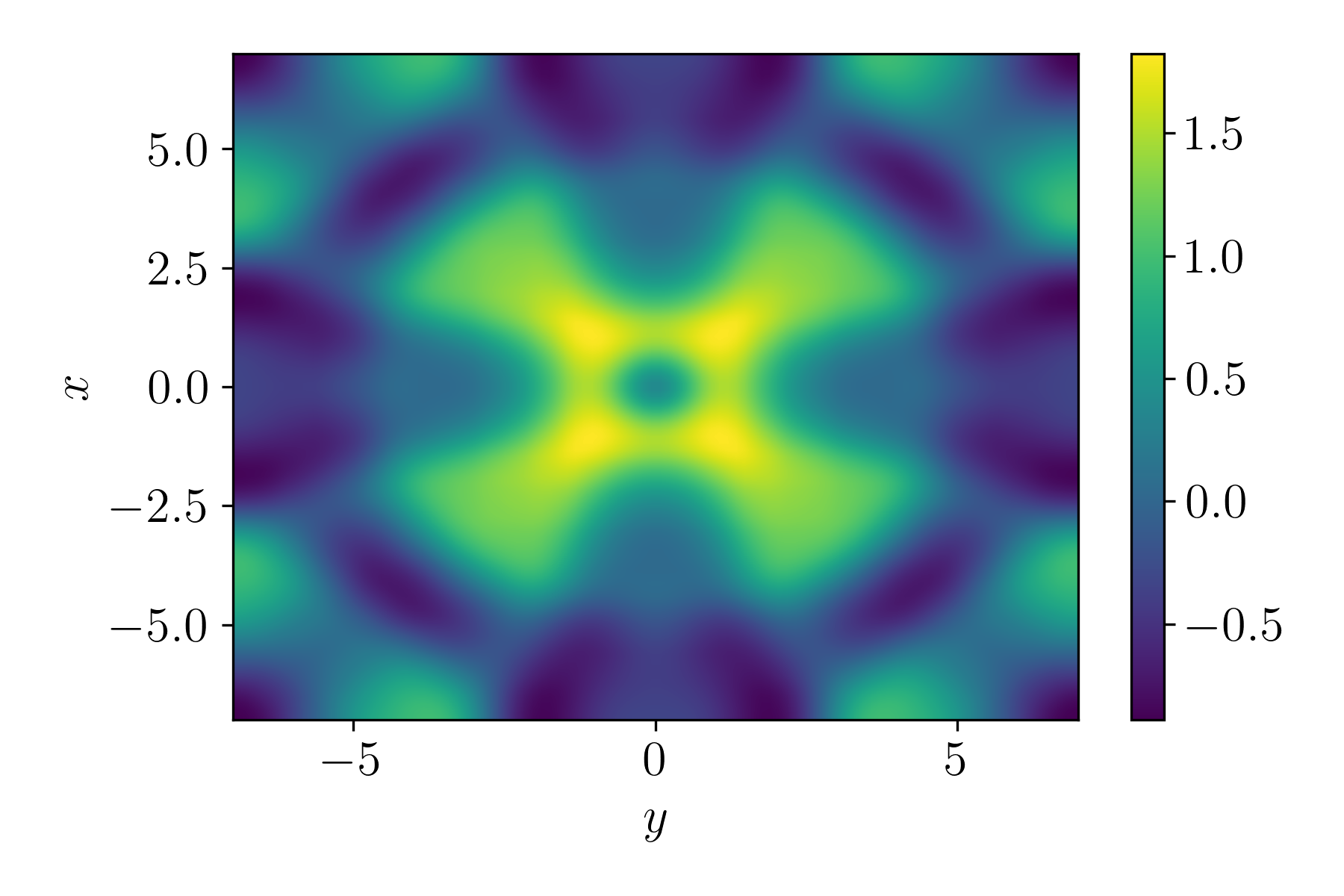}
		\caption{Ground truth}	
	\end{subfigure}
	\begin{subfigure}{0.328\textwidth}
		\includegraphics[width=1\linewidth]{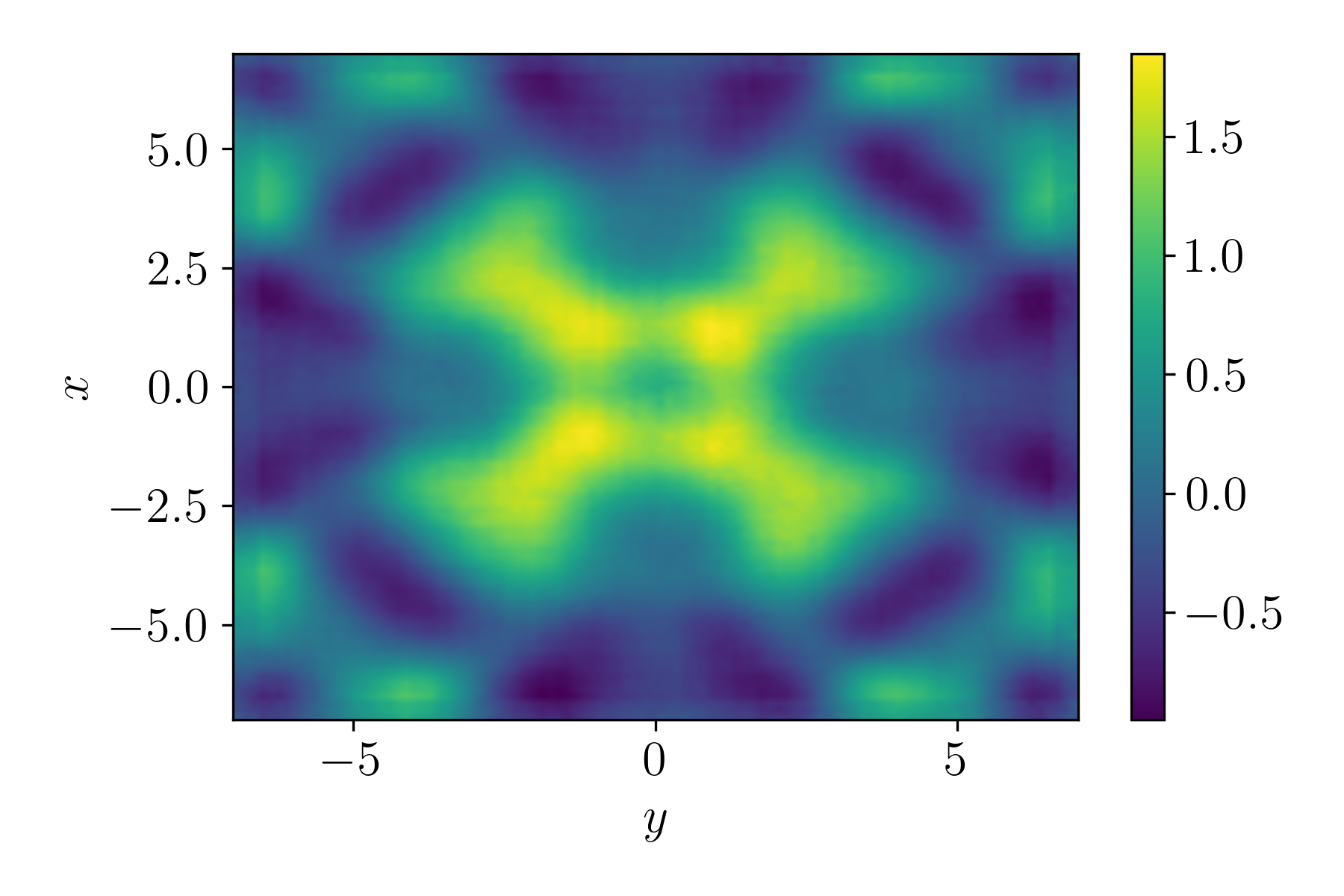} 
		\caption{SympCAE}
	\end{subfigure}
	\begin{subfigure}{0.328\textwidth}
		\includegraphics[width=1\linewidth]{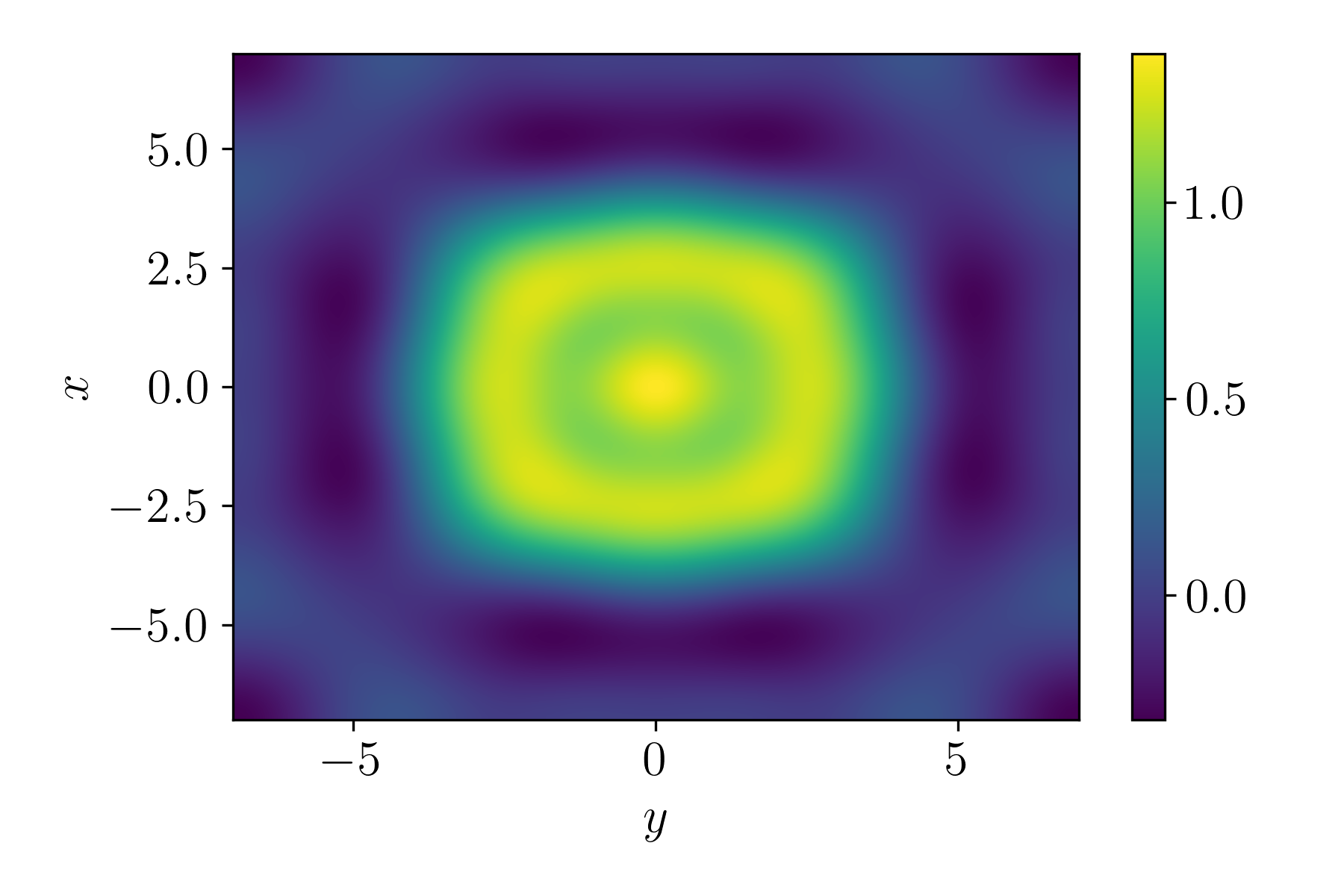}
		\caption{PSD}
	\end{subfigure}
	\caption{SG equation: Comparison of the solutions of $ p $ at final time $ t = 20 $ for latent time dimension $ r = 3 $. (a) Ground truth, (b) SympCAE, (c) PSD. }
	\label{fig:SG-v}
\end{figure}

\section{Conclusions}\label{sec:conc}
We have proposed a nonlinear symplectic convolutional autoencoder by utilizing symplectic neural networks (SympNet) and proper symplectic decomposition (PSD). The main application of our method is the transformation of Hamiltonian systems into an equivalent, low-dimensional, Hamiltonian form. Nevertheless, the method is not limited to Hamiltonian systems; many physical phenomena that require volume-preserving properties can be modeled via the proposed method. We compared our method with PSD in terms of data compression and reconstruction. We demonstrated the generality of the method by combining it with a SympNet after dimensionality reduction.

Some future work motivated by this study includes testing the symplectic convolutional autoencoder with noisy data and extending the framework to a 3D symplectic convolutional autoencoder case.

\section*{Acknowledgment} 
Süleyman Y\i ld\i z and Konrad Janik would like to thank Jens Saak for fruitful discussions.
\subsection*{Funding Statement}
Süleyman Y\i ld\i z and Peter Benner are partially supported by the German Research Foundation (DFG) Research Training Group 2297 ``MathCoRe'', Magdeburg.

\section*{Data Availability Statement}
Data and relevant code for this research work have been archived within the Zenodo repository \cite{YilJB25software}.

\section*{Appendix}
\subsection*{Models}
In this subsection, we explain the architecture of the models used in \Cref{sec:num}, how we trained them and which choices of hyperparameters we made.\\
The general architecture of our SympCAEs is shown in \Cref{fig:CNN}. The convolutional block does not only consist of convolutional layers, but of activation layers as well. The architectures of the encoder and decoder in the SympCAE autoencoder exhibit mirror symmetry, possessing a structure akin to a symplectic inverse. Consequently, we present the exact number of convolutional and activation layers used in the encoder for both the 1D and 2D SympCAE in \Cref{tab:CNN}. Note that the composition of two activation layers without a convolutional layer in-between does not violate \Cref{def:enc}, since the in-between convolutional layer is just chosen to be the identity. The SympCAEs $\psi \in \Psi_{\text{AE}}$ are trained with a standard autoencoder loss with $L_2$-regularization
\begin{align}
	\label{eqn:CNNloss}
	\mathcal{L} = \norm{\psi(X)-X}_2^2 + \lambda_2 \sum_{\theta \in \Theta} \norm{\theta}_2,
\end{align}
where $X$ is the snapshot matrix from the training set and $\Theta$ is the set of trainable parameters of the SympCAE $\psi$. Hyperparameters and other training details for the SympCAE can be found in \Cref{tab:CNN} as well.\\
\begin{figure}[tb]
	\centering
	\includegraphics[width=\textwidth]{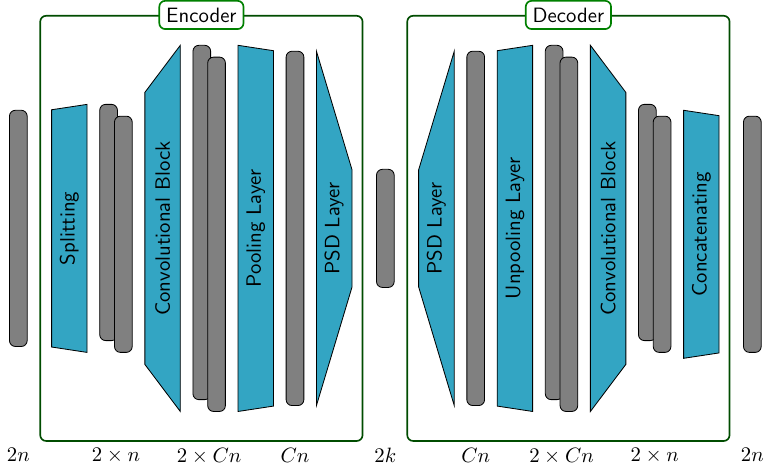}
	\caption{Schematic description of the 1D symplectic autoencoder architecture used for the wave equation and NLS examples. $2n$ is the full state dimension, while $2k$ describes the size of the latent space. $C=\Pi_{i}\frac{C_{\text{out}}^{(i)}}{C_\text{in}^{(i)}}$ is the product of the quotients of the output and input channels of all the convolutional layers}
	\label{fig:CNN}
\end{figure}
\begin{table}[tb]
	\centering
	\renewcommand{\arraystretch}{1.2}
	\begin{tabular}{ccc}
		\textbf{setting}  &  1D SympCAE & 2D SympCAE \\ \hline
		convolutional layers & 12 & 4 \\ 
		convolutional channels & [2, 4, 4, 4, 2, 4, 4, 4, 2, 4, 4, 4] & [2, 4, 4, 8]\\ 
		kernel size $l$ & 21 & 7 $ \times$ 7 \\ 
		stride & 1 & 1 \\ 
		padding & "zero" padding of size $(l-1)/2$ & "zero" padding of size $(l-1)/2$ \\ 
		activation layers & 2 after every 4 convolutional layers & 2 after convolutional layers \\ 
		batch size & $N_t$/$2$  & $N_t$/$2$   \\ 
		learning rate & $10^{-3}$ & $10^{-3}$ \\
		epochs & $6000$ & $6000$  \\ 
		$\lambda_2$ from \cref{eqn:CNNloss} & $10^{-5}$ & $10^{-5}$ \\ \hline
        
	\end{tabular}
	\caption{CNN architecture and hyperparameters for the encoder of the 1D and 2D SympCAE}
	\label{tab:CNN}
\end{table}
The SympNets $\phi \in \Psi_\text{LA}$ from \Cref{subsec:wave,subsec:NLS} are optimized with the following loss function
\begin{align*}
	\mathcal{L}= \norm{\phi(x) - x'}_2^2,
\end{align*}
where $x=\psi_{\text{Enc}}(X)$ is the latent representation of $X$ obtained by the encoder part $\psi_{\text{Enc}}\in \Psi_{\text{Enc}}$ of the SympCAE $\psi$ and $x'=\psi_{\text{Enc}}(X')$. $X'$ is the training set $X$ shifted in time by one step. The used architecture, hyperparameters and training details can be found in \Cref{tab:SympNet}. To train all the symplectic neural networks, we utilize \texttt{PyTorch} \cite{PasGMetal19} with the Adam algorithm \cite{KinB17}, in combination with a \texttt{StepLR} scheduler, using the \texttt{PyTorch Lightning} \cite{lightning} module. All the models are trained on a machine with an \texttt{Intel\textsuperscript{\tiny\textcopyright} Core\textsuperscript{\tiny TM} i5-12600K} CPU and \texttt{NVIDIA RTX\textsuperscript{\tiny TM} A4000(16GB)} GPU. 

\begin{table}[tb]
	\centering
	\renewcommand{\arraystretch}{1.2}
	\begin{tabular}{cc}
		\textbf{setting}  &  SympNet \\ \hline
		layers & 8 \\ 
		sublayers & 1 \\ 
		optimizer & Adam (weight decay $10^{-6}$) \\ 
		learning rate & $10^{-1}$ \\ 
		learning rate scheduler & StepLR \\ \hline
	\end{tabular}
	\caption{SympNet architecture and hyperparameters for wave equation and NLS example}
		\label{tab:SympNet}
\end{table}
\clearpage
\addcontentsline{toc}{section}{References}
\bibliographystyle{ieeetr}
\bibliography{exampleref}

\end{document}